\documentclass{article}
\usepackage{fullpage}
\usepackage{smile}
\RequirePackage{natbib}

\usepackage[pdftex,colorlinks,linkcolor=blue,citecolor=blue,filecolor=blue,urlcolor=blue]{hyperref}      
\usepackage{url}           
\usepackage{booktabs}      
\usepackage{amsfonts}       
\usepackage{nicefrac}       
\usepackage{xcolor}        
\usepackage[capitalize,noabbrev]{cleveref}

\usepackage{subfigure}
\usepackage{multirow}
\usepackage{color}
\usepackage{amsmath, amsthm, amssymb, amsfonts}

\usepackage{xspace}
\usepackage{comment}
\usepackage{dsfont}
\usepackage{subfigure}
\usepackage{algorithm}
\usepackage{algorithmic}
\usepackage{multirow}
\RequirePackage{times}
\usepackage{comment}
\usepackage{xcolor}

\usepackage{authblk}
\author[1]{Yunfan Li}
\author[2]{Yiran Wang}
\author[3]{Yu Cheng}
\author[4]{Lin Yang}

\affil[1,2,4]{University of California, Los Angeles}
{
	\makeatletter
	\renewcommand\AB@affilsepx{, \protect\Affilfont}
	\makeatother
	\affil[1]{yunfanli@g.ucla.edu}
        \affil[2]{yiranwang1027@g.ucla.edu}
	\affil[4]{linyang@ee.ucla.edu}
	
}
\affil[3]{Microsoft Research}
{
	\makeatletter
	\renewcommand\AB@affilsepx{, \protect\Affilfont}
	\makeatother
	\affil[3]{yu.cheng@microsoft.com}

}

\title{Low-Switching Policy Gradient with Exploration \\ via Online Sensitivity Sampling}

\begin{document}

	\date{}
	\maketitle

	\begin{abstract}
		Policy optimization methods are powerful algorithms in Reinforcement Learning (RL) for their flexibility to deal with policy parameterization and ability to handle model misspecification. However, these methods usually suffer from slow convergence rates and poor sample complexity. Hence it is important to design provably sample efficient algorithms for policy optimization. Yet, recent advances for this problems have only been successful in tabular and linear setting, whose benign structures cannot be generalized to non-linearly parameterized policies. In this paper, we address this problem by leveraging recent advances in value-based algorithms, including bounded eluder-dimension and online sensitivity sampling, to design a low-switching sample-efficient policy optimization algorithm, \emph{LPO}, with general non-linear function approximation. We show that, our algorithm obtains an $\varepsilon$-optimal policy with only  $\widetilde{O}(\frac{\text{poly}(d)}{\varepsilon^3})$  samples, where $\varepsilon$ is the suboptimality gap and  $d$ is a complexity measure of the function class approximating the policy. This drastically improves previously best-known sample bound for policy optimization algorithms, $\widetilde{O}(\frac{\text{poly}(d)}{\varepsilon^8})$. Moreover, we empirically test our theory with deep neural nets to show the benefits of the theoretical inspiration.
		
	\end{abstract}
	
	\section{Introduction}
Reinforcement learning (RL) has achieved great success in many practical areas by adopting policy gradient methods with deep neural networks \citep{schulman2015trust, schulman2017proximal, haarnoja2018soft}. These policy optimization methods are some of the most classic \citep{williams1992simple,  konda1999actor} approaches for RL. Although their theoretical convergence properties have been established in \citep{geist2019theory, abbasip, agarwal2020optimality, bhandari2019global} with assumptions that the state space is already well-explored, it is usually not the case in practice. 
To resolve this issue, policy-based approaches with active exploration in the environment have been proposed in simple tabular~\citep{shani2020optimistic}, linear function approximation~\citep{cai2020provably, agarwal2020pc} and general function approximation \citep{feng2021provably} models.

Among these exploration-based approaches, \citet{agarwal2020pc} and \citet{feng2021provably} are specially designed to handle model-misspecification more robustly than existing value-based approaches \citep{jin2020provably,wang2020reinforcement} by performing policy gradient methods to solve a sequence of optimistic MDPs. However, the robustness of both \citep{agarwal2020pc} and \citep{feng2021provably} pays a huge price: to obtain an $\varepsilon$-suboptimal policy, \citet{agarwal2020pc} requires $\sim\widetilde{O}(1/\varepsilon^{11})$, and \citet{feng2021provably} requires $\sim\widetilde{O}(1/\varepsilon^8)$ number of samples to obtain an $\varepsilon$-optimal policy. Recently, \citet{zanette2021cautiously} has designed a low switching (i.e. reducing the number of policy changes) policy-based algorithm with linear function approximation, which largely reduces the sample complexity of \citet{agarwal2020pc}. However, it is still unknown how to improve sample complexity of policy-based algorithms with good robustness in the non-linear setting.

As for the value-based methods, low-switching techniques \citep{bai2019provably, gao2021provably, kong2021online} are utilized to reduce the policy changes of the algorithm. Among them, \citet{kong2021online} proposed a novel notion of online sensitivity score, which measures the importance of a data point relative to a given dataset over some \emph{general} function class. By using this sensitivity score, \citet{kong2021online} established an online sub-sampling technique which greatly reduced the average \emph{computing time} of previous work \citep{wang2020reinforcement}.
Nevertheless, it is unknown whether such low-switching techniques can be applied to save \emph{samples} in policy-based approaches.

In this paper, we present a low-switching policy-based algorithm \textbf{LPO} (\emph{\textbf{L}ow-Switching \textbf{P}olicy Gradient and Exploration via \textbf{O}nline Sensitivity Sampling}), which leverages techniques in policy-based approaches, such as \citep{feng2021provably, zanette2021cautiously} and value-based approach, such as \citep{kong2021online} to establish efficient policy gradient on non-linear function class while preserving the low-switching property to save samples and running time.
Our algorithm follows an actor-critic framework, where the critic guides the exploration of the policy via exploration bonuses derived from the non-linear function class, and  policy-gradient (PG) updates the policy to guarantee robustness and stability. 
The low-switching technique is applied primarily to derive a slowly updating critic, while preserving the quality of learning.
Since one of the major terms in sample complexity originates from the PG policy update, slowly updating critic can drastically save the sample complexity as it requires only a few policy updates.
Concretely, our approach only update the policy for $\sim \log T$ times for running $T$ rounds of the algorithm, whereas existing approaches, e.g., \citep{feng2021provably}, which also targets on the policy-based exploration with non-linear function approximation, takes at least $\sim T$ policy updates. 
We also derive new PG approaches aware of the structure of non-linear function class to further save samples in updating the policy.

\paragraph{Our Contribution}

\begin{itemize}
    \item 
    We design a new policy-based exploration algorithm, \textbf{LPO}, with non-linear function approximation. 
    The algorithm enjoys the same stability guarantee in terms of model-misspecification as presented in existing approaches \citep{feng2021provably}.
    This algorithm leverages efficient value-based techniques (online sensitivity sampling) to slowly update its policy and thus enjoys a sample complexity of $\widetilde{O}(\text{poly}(d)/\varepsilon^3)$, whereas existing approach takes at least $\widetilde{O}(\text{poly}(d)/\varepsilon^8)$ samples to obtain an $\varepsilon$-optimal policy, where $d$ is related to the eluder-dimension, measuring the complexity of the function class.
    
    \item 
    While enjoying a theoretical guarantee at special cases where the function class has a bounded complexity, the algorithm itself can be implemented using neural networks. 
    We further empirically tested the theoretical inspiration of online sensitivity sampling with existing deep RL frameworks. The experimental results demonstrated the efficacy of our approaches.

\end{itemize}

\paragraph{Related Work}
With regards to exploration methods in RL, there are many provable results in the tabular case \citep{kearns2002near, brafman2002r, kearns89, jin2018q} and linear \citep{yang2019sample,yang2020reinforcement, jin2020provably} settings with value or model-based methods. Recent papers \citep{shani2020optimistic, cai2020provably, agarwal2020pc} have developed policy-based methods also in tabular and linear settings and \citet{zanette2021cautiously} greatly reduces the sample complexity of \citet{agarwal2020pc} mainly by using the doubling trick for determinant of empirical cumulative covariance. However, relative few provable results are achieved in non-linear setting.

For general function approximation, complexity measures are essential for non-linear function class, and \citet{russo2013eluder} proposed the concept of eluder dimension. Recent papers have extended it to more general framework (e.g. Bellman Eluder dimension \citep{jin2021bellman}, Decision-Estimation Coefficient \citep{foster2021statistical}, Admissible Bellman Characterization \citep{chen2022general}). However, the use of eluder dimension allows computational tractable optimization methods. Based on the eluder dimension, the value-based technique of \citet{wang2020reinforcement} describes a UCB-VI style algorithm that can explore the environment driven by a well-designed width function and \citet{kong2021online} devises an online sub-sampling method which largely reduces the average computation time of \citet{wang2020reinforcement}.

For policy-based method in the general setting, \citet{feng2021provably} proposes a model-free algorithm with abundant exploration to environment using the indicator of width function. Moreover, it has better robustness to model misspecification compared to (misspecified) Linear MDP \citep{jin2020provably}. However, \citet{feng2021provably} suffers from huge sample complexity. In this paper, instead of directly finding a similar notion in the non-linear setting just like determinant in linear setting \citep{zanette2021cautiously}, we adopt an online sensitivity-sampling method to quantify the sensitivity of new-coming data obtained from the environment. Moreover, the importance of designing a sophisticated and efficient reward bonus is mentioned in \citep{zanette2021cautiously} and we significantly generalize this approach to the non-linear setting by combining the width function and its indicator and our reward bonuses save samples and computing time compared to \citep{feng2021provably}.

\paragraph{Notations} We use $[n]$ to represent index set $\{1,\cdots n\}$. For $x \in \mathbb{R}$, $\lfloor{x\rfloor}$ represents the largest integer not exceeding $x$ and $\lceil{x\rceil}$ represents the smallest integer exceeding $x$. Given $a,b \in \mathbb{R}^d$, we denote by $a^\top b$ the inner product between $a$ and $b$ and $||a||_2$ the Euclidean norm of $a$. Given a matrix $A$, we use $||A||_2$ for the spectral norm of $A$, and for a positive definite matrix $\Sigma$ and a vector $x$, we define $||x||_{\Sigma}=\sqrt{x^\top \Sigma x}$. We abbreviate Kullback-Leibler divergence to \textbf{KL} and use $O$ to lead orders in asymptotic upper bounds and $\widetilde{O}$ to hide the polylog factors. For a finite set $\mathcal{A}$, we denote the cardinality of $\mathcal{A}$ by $|\mathcal{A}|$, all distributions over $\mathcal{A}$ by $\Delta (\mathcal{A})$, and especially 
the uniform distribution over $\mathcal{A}$ by $\text{Unif} (\mathcal{A})$.

For a function $f: \mathcal{S} \times \mathcal{A} \rightarrow \mathbb{R}$, define

$$
\|f\|_{\infty}=\max _{(s, a) \in \mathcal{S} \times \mathcal{A}}|f(s, a)| .
$$

Similarly, for a function $v: \mathcal{S} \rightarrow \mathbb{R}$, define

$$
\|v\|_{\infty}=\max _{s \in \mathcal{S}}|v(s)| .
$$

For a set of state-action pairs $\mathcal{Z} \subseteq \mathcal{S} \times \mathcal{A}$, for a function $f: \mathcal{S} \times \mathcal{A} \rightarrow \mathbb{R}$, we define the $\mathcal{Z}$-norm of $f$ as

$$
\|f\|_{\mathcal{Z}}=\left(\sum_{(s, a) \in \mathcal{Z}}(f(s, a))^{2}\right)^{1 / 2}.
$$

\section{Preliminaries} \label{sec 2}
\paragraph{Markov Decision Process} In this paper, we consider discounted Markov decision process (MDP) environment, which can be specified by a tuple, $\mathcal{M}\ =\ (\mathcal{S}, \mathcal{A}, P, r, \gamma)$, where $\mathcal{S}$ is a possibly infinite state space, $\mathcal{A}$ is a finite action space and we denote $A=|\mathcal{A}|$, $P: \mathcal{S} \times \mathcal{A} \rightarrow \Delta(\mathcal{S})$ specifies a transition kernel and is unknown to the learner, $r: \mathcal{S} \times \mathcal{A} \rightarrow[0,1]$ is a reward function, and $\gamma \in(0,1)$ is a discount factor that discounts the reward received in a future time step.

Suppose an RL agent chooses an action $a\in\mathcal{A}$  at the current state $s$, 
the environment brings the agent to a new state $s'$ with the unknown probability $P\left(s^{\prime} \mid s, a\right)$ and the agent receives an instant reward $r(s, a)$. The goal for a leaner is to find a policy\footnote{We here only consider stationary policies as one can always find a stationary optimal-policy in a discounted MDP \citep{puterman2014markov}.} $\pi: \mathcal{S} \rightarrow \Delta(\mathcal{A})$ such that the expected long-term rewards are maximized. In particular, the quality of a policy can be measured by the
the $Q$-value function $Q^{\pi}: \mathcal{S} \times \mathcal{A} \rightarrow \mathbb{R}$ is defined as:

$$
Q^{\pi}(s, a):=\mathbb{E}^{\pi}\left[\sum_{t=0}^{\infty} \gamma^{t} r\left(s_{t}, a_{t}\right) \mid s_{0}=s, a_{0}=a\right],
$$

where the expectation is taken over the trajectory following $\pi$ -- this measures the expected discounted total returns of playing action $a$ at state $s$ and then playing policy $\pi$ (for an indefinite amount of time). And after taking expectation over the action space, we get the value function: $V^{\pi}(s):=$ $\mathbb{E}_{a \sim \pi(\cdot \mid s)}\left[Q^{\pi}(s, a)\right]$, which measures the total expected discounted returns of playing policy $\pi$ starting from state $s$. From $V^{\pi}$ and $Q^{\pi}$, we can further define the advantage function of $\pi$ as $A^{\pi}(s, a)=Q^{\pi}(s, a)-$ $V^{\pi}(s)$, which measures whether the action $a$ can be further improved. Moreover, if a policy $\pi^*$ is optimal, then the Bellman equation \citep{puterman2014markov} states that 
$A^{\pi^*}(s,a)\le 0$ for all $s, a$ and $\mathbb{E}_{a\sim \pi^*(\cdot|s)}[A^{\pi^*}(s, a)] = 0$.
In practice, we may also restrict the policy space being considered as $\Pi$ (which may be parameterized by a function class).

We also define the discounted state-action distribution $d_{\tilde{\tilde{s}}}^{\pi}(s, a)$ induced by $\pi$ as:

$$
d_{\tilde{s}}^{\pi}(s, a)=(1-\gamma) \sum_{t=0}^{\infty} \gamma^{t} \operatorname{Pr}^{\pi}\left(s_{t}=s, a_{t}=a \mid s_{0}=\tilde{s}\right),
$$

where $\operatorname{Pr}^{\pi}\left(s_{t}=s, a_{t}=a \mid s_{0}=\tilde{s}\right)$ is the probability of reaching $(s, a)$ at the $t_{\text {th }}$ step starting from $\tilde{s}$ following $\pi$. Similarly, the definition of  $d_{\tilde{s}, \tilde{a}}^{\pi}(s, a)$ can be easily derived as the distribution of state-actions if the agent starts from state $\tilde{s}$ and selects an action $\tilde{a}$. For any initial state-actions distribution $\nu \in \Delta(\mathcal{S} \times \mathcal{A})$, we denote by $d_{\nu}^{\pi}(s, a):=\mathbb{E}_{(\tilde{s}, \tilde{a}) \sim \nu}\left[d_{(\tilde{s}, \tilde{a})}^{\pi}(s, a)\right]$ and $d_{\nu}^{\pi}(s):=\sum_{a} d_{\nu}^{\pi}(s, a)$. Given an initial state distribution $\rho \in \Delta(\mathcal{S})$, we define $V_{\rho}^{\pi}:=\mathbb{E}_{s_{0} \sim \rho}\left[V^{\pi}\left(s_{0}\right)\right]$. With these notations, the reinforcement learning (RL) problem with respect to the policy class $\Pi$ is reduced to solving the following optimization problem. 
$$
\underset{\pi \in \Pi}{\operatorname{maximize}} V_{\rho_{0}}^{\pi},
$$
for some initial distribution $\rho_0$. 
We  further, without loss of generality \footnote{Otherwise, we can modify the MDP and add a dummy state $s_0$ with $\rho_0$ as its state transition for all actions played at $s_0$.}, assume $\rho_0$ is a singleton on some state $s_{0}$.

\paragraph{Policy Space and Width Function}We now formally define the policy parameterization class, which is compatible with a neural network implementation.
For a set of functions $\mathcal{F} \subseteq\{f: \mathcal{S} \times \mathcal{A} \rightarrow \mathbb{R}\}$,  we consider a policy space as $\Pi_{\mathcal{F}}:= \{ \pi_f, f\in\mathcal{F}\}$
by applying the softmax transform to functions in $\mathcal{F}$, i.e., for any $f\in \mathcal{F}$,
$$
\pi_f(a|s) := \frac{\exp(f(s,a))}{\sum_{a'\in\mathcal{A}}\exp(f(s,a'))}
$$

Given $\mathcal{F}$, we define its function difference class $\Delta \mathcal{F} := \{\Delta f|\  \Delta f = f- f', f,f'\in \mathcal{F}\}$ and width function on $\Delta \mathcal{F}$ for a state-action pair $(s, a)$ as \label{width}

$$
w(\Delta \mathcal{F}, s, a)=\sup _{\Delta f \in \Delta \mathcal{F} }|\Delta f(s, a)|.
$$
As we will show shortly, this width function will be used to design exploration bonuses for our algorithm.

\section{Algorithms} \label{sec 3}

\begin{algorithm}[t]
\caption{\textbf{LPO}} \label{A1}
\begin{algorithmic}[1]
\STATE \textbf{Input}: Function class $\mathcal{F}$
\STATE \textbf{Hyperparameters}: $N,\delta,\beta$ 
\STATE For all $s \in S$, initialize $\pi^{0}(\cdot|s)=\text{Unif}(\mathcal{A})$, $\widehat{\mathcal{Z}}^1=\emptyset$
\FOR{$n=1,2,\cdots,N$} 
\STATE Update policy cover $\pi_{cov}^{n}=\pi^{0:n-1}$
	 \STATE ${\widehat{\mathcal{Z}}}^n \leftarrow \textbf{S-Sampling} \big{(}\mathcal{F},\widehat{\mathcal{Z}}^{n-1},(s_{n-1},a_{n-1}),\delta\big{)}$
\IF{${\widehat{\mathcal{Z}}}^n \neq {\widehat{\mathcal{Z}}}^{\underline{n}}$ \textbf{or} $n=1$}
\STATE Update the known set and bonus function
                      \STATE ${\mathcal{K}}^n=\{(s,a)\ |\ \omega({\widehat{\mathcal{F}}}^n,s,a)< \beta\}$
                      \STATE $b^n(s,a)=\frac{3}{1-\gamma}\cdot\textbf{1}\{\omega({\widehat{\mathcal{F}}}^n,s,a)\geq \beta\}+\frac{2}{\beta}\cdot\omega({\widehat{\mathcal{F}}}^n,s,a)\cdot\textbf{1}\{\omega({\widehat{\mathcal{F}}}^n,s,a) <\beta\}$            
                      \STATE Set $\underline{n}\leftarrow n$
            \STATE $\pi^n\leftarrow \textbf{Policy Update}(\pi_{cov}^n,b^n,\mathcal{K}^n)$
\ELSE
\STATE $\pi^n \leftarrow \pi^{\underline{n}},\mathcal{K}^n \leftarrow \mathcal{K}^{\underline{n}},b^n \leftarrow b^{\underline{n}}$
   \ENDIF
   \STATE $(s_n,a_n)\leftarrow \textbf{d-sampler}(\pi^{\underline{n}},\nu)$
\ENDFOR
\STATE \textbf{Output}: Unif($\pi^0,\pi^1,\cdots,\pi^{N-1}$)
\end{algorithmic}
\end{algorithm}

\begin{algorithm}[t]
\caption{\textbf{S-Sampling (Sensitivity-Sampling)}} \label{A2}
\begin{algorithmic}[1]
\STATE \textbf{Input}: Function class $\mathcal{F}$, current sensitivity dataset $\widehat{\mathcal{Z}}$, a new state-action pair $z$, failure probability $\delta$.
\STATE Compute the sensitivity factor of $z$ relative to the dataset $\mathcal{Z}$:
$$
s_z=C\cdot\text{sensitivity}_{\widehat{\mathcal{Z}},\mathcal{F}}(z)\cdot\text{log}(N\mathcal{N}(\mathcal{F},\sqrt{\delta/64N^3})/\delta)
$$
\STATE Let $\widehat{z}$ be the neighbor of $z$ satisfying the condition \eqref{E2}
\IF {$s_z\geq 1$}
    \STATE Add 1 copy of $\widehat{z}$ into $\mathcal{Z}$
\ELSE
    \STATE Let $n_z = \lfloor{\frac{1}{s_z}\rfloor}$ and add $n_z$ copies of $\widehat{z}$ into $\widehat{\mathcal{Z}}$ with probability $\frac{1}{n_z}$ 
\ENDIF
\STATE \textbf{return} $\widehat{\mathcal{Z}}$.

\end{algorithmic}
\end{algorithm}

\begin{algorithm}[tb]
    \caption{\textbf{Policy Update}} \label{A3}
\begin{algorithmic}[1]
    \STATE \textbf{Input:} $\pi_{cov},b,\mathcal{K}$
    \STATE \textbf{Initialize:}$\ \pi_0(\cdot|s)=\text{Unif}(\mathcal{A})\  \text{if}\  s \in \mathcal{K}\  and$
    \STATE $\ \ \ \ \ \ \ \ \ \ \ \ \ \ \ \pi_0(\cdot|s)=\text{Unif}(\{a|(s,a)\notin\mathcal{K}\})\  \text{if}\  s\notin\mathcal{K} $
    
    \FOR{$k=1,2,\cdots,K-1$} 
	    \IF{$k-\underline{k}>\kappa$ or $k=0$}
	       \STATE $\underline{k} \leftarrow k$
	       \STATE $D \leftarrow \textbf{Behaviour Policy Sampling}(\pi_{\underline{k}},\pi_{\text{cov}})$
	       \ENDIF
	\STATE $\widehat{Q}_k \leftarrow$ \textbf{Policy Evaluation Oracle}$(\pi_k,D,\pi_{\underline{k}},b)$
	\STATE Update policy: $\forall s\in \mathcal{K}$
	\STATE $\pi_{k+1}(\cdot|s)\propto\pi_{k}(\cdot|s)e^{\eta\widehat{Q}_k(\cdot|s)}$
    \ENDFOR
\STATE \textbf{return:}$\ \pi_{0:K-1}=\text{Unif}\{\pi_0,\cdots,\pi_{K-1}\}$
\end{algorithmic}
\end{algorithm}


In this section, we present our algorithm \emph{Low-Switching Policy Gradient and Exploration via Online Sensitivity Sampling }\textbf{(LPO)}. The algorithm takes a function class $\mathcal{F}$ as an input and interacts with the environment to produce a near-optimal policy. The complete pseudocode is in \cref{A1}. We first give an overview of our algorithm before describing the details of our improvements.

\subsection{Overview of our Algorithm}
Our algorithm \textbf{LPO} (\cref{A1}) has two loops. The outer loop produces a series of well-designed optimistic MDPs by adding a reward bonus and choosing an initial state distribution which are then solved with regression in the inner loop. These optimistic MDPs will encourage the agent to explore unseen part of the environment. In our \textbf{LPO}, we construct the initial state distribution by using the uniform mixture of previous well-trained policies (also called policy cover).

Specifically, at the beginning of $n$-th iteration, we have already collected sample $(s_n,a_n)$ using the last policy $\pi^{\underline{n}}$. Then at iteration $n$, we use \textbf{S-Sampling} (i.e. \textbf{Sensitivity-Sampling}) (\cref{A2}) to measure the change that the new sample brings to the dataset. If the current sample can provide sufficiently new information relative to the formal dataset, then with great probability, we choose to store this data and invoke the inner loop to update the policy. Otherwise, we just abandon this data and continue to collect data under $\pi^{\underline{n}}$. Through this process, a policy cover $\pi_{cov}^n=$ Unif($\pi^0,\pi^1,\cdots,\pi^{n-1}$) is constructed to provide an initial distribution for the inner routine. To this end, we define the known state-actions $\mathcal{K}^n$, which can be visited with high probability under $\pi_{cov}^n$. Using $\mathcal{K}^n$ and a reward bonus $b^n$, we create an optimistic MDP to encourage the agent to explore outside $\mathcal{K}^n$ as well as to refine its estimates inside $\mathcal{K}^n$.

In the inner routine, the algorithm \textbf{Policy Update} (\cref{A3}) completes the task to find an approximately optimal policy $\pi^n$ in the optimistic MDP through general function approximation. This policy $\pi^n$ would produce new samples which will be measured in the next iteration. Under the procedure of our algorithm, the policy cover will gain sufficient coverage over the state-action space and the bonus will shrink. Therefore, the near-optimal policies in the optimistic MDPs eventually behave well in the original MDP. Next, we will describe the details of each part of our algorithm.

\subsection{Outer Loop}
Now we describe the details of three important parts in the outer loop.

\paragraph{Lazy Updates of Optimistic MDPs via Online Sensitivity-Sampling} The lazy or infrequent updates of the optimistic MDPs in \textbf{LPO} play a crucial role of improving sample complexity, which reduce the number of \textbf{Policy Update} invocations from $O(N)$ to $O(\text{poly}(\log N))$. For the linear case, \citet{zanette2021cautiously} achieves this result by monitoring the determinant of the empirical cumulative covariance matrix. However, in our general setting, we can not count on the linear features anymore. Instead, we introduce our online sensitivity sampling technique, which is also mentioned in \cite{wang2020reinforcement, kong2021online}. 

Now we describe the procedure for constructing the sensitivity dataset $\widehat{\mathcal{Z}}^n$. At the beginning of iteration $n$, the algorithm receives the current sensitivity dataset $\widehat{\mathcal{Z}}^{n-1}$ and the new data $(s_{n-1},a_{n-1})$ from the last iteration. We first calculate the online sensitivity score in \eqref{E1} to measure the importance of $(s_{n-1},a_{n-1})$ relative to $\widehat{\mathcal{Z}}^{n-1}$.

\begin{equation}\label{E1}
\begin{aligned}
&\operatorname{sensitivity}_{\mathcal{Z}^{n}, \mathcal{F}}\left(z\right) \\  =&\mathop{\text{sup}}\limits_{f_1,f_2\in \mathcal{F}}\frac{{\left(f_1(z)-f_2(z)\right)}^2}{\text{min}\{{||f_1-f_2||}_{\mathcal{Z}^n}^2,4NW^2\}+1}
\end{aligned}
\end{equation}

Then the algorithm adds $(s_{n-1},a_{n-1})$ to $\widehat{\mathcal{Z}}^{n-1}$ with probability decided by its online sensitivity score. Intuitively, the more important the new sample is, the more likely it is to be added. At the same time, the algorithm set the number of copies added to the dataset according to the sampling probability, if added. In addition, due to the technical obstacle in theoretical proof, we need to replace the original data $z$ with the data $\widehat{z}\in\mathcal{C}(\mathcal{S}\times\mathcal{A},1/16\sqrt{64N^3/\delta})$ in the $\varepsilon$-cover set (defined in Assumption \ref{ass 4.2}), which satisfies:

\begin{equation}\label{E2}
\mathop{\text{sup}}\limits_{f\in\mathcal{F}}|f(z)-f(\widehat{z})|\leq1/16\sqrt{64N^3/\delta}  
\end{equation}

Furthermore, the chance that the new sample is sensitive gets smaller when the dataset gets enough samples, which means that the policy will not change frequently in the later period of training. As will be demonstrated later, the number of switches (i.e. the number of policy changes in the running of the outer loop) achieve the logarithmic result. To this end, the size of sensitivity dataset is bounded and provides good approximation to the original dataset, which serves as a benign property for theoretical proof.

\paragraph{Known and Unknown state-actions}
According to the value of width function (defined in \cref{width}) under the sensitivity dataset, the state-action space $\mathcal{S}\times\mathcal{A}$ is divided into two sets, namely the known set $\mathcal{K}^n$ in \eqref{E3} and its complement the unknown set. 

\begin{equation}\label{E3}
	\begin{aligned}
	\mathcal{K}^n=&\{(s,a)\in\mathcal{S}\times\mathcal{A}|\ \omega(\widehat{\mathcal{F}}^n,s,a)<\beta\}
	\end{aligned}
\end{equation}

where
$$
{\widehat{\mathcal{F}}}^n=\{\Delta f \in \Delta \mathcal{F}|\  ||\Delta f||_{\widehat{\mathcal{Z}}^n}^2\leq \epsilon \}
$$
In fact, the width function $\omega(\widehat{\mathcal{F}}^n,s,a)$ serves as a prediction error for a new state-action pair $(s,a)$ where the training data is $\widehat{\mathcal{Z}}^n$, which is the general form of $||\phi(s,a)||_{(\widehat{\Sigma}^n)^{-1}}$ in the linear case. Therefore, the known set $\mathcal{K}^n$ represents the state-action pairs easily visited under the policy cover $\pi_{\text{cov}}^n$. If all the actions for one state are in the $\mathcal{K}^n$, we say this state is known.

\begin{equation}
\mathcal{K}^n=\{s\in\mathcal{S}|\ \forall a\in\mathcal{A},\omega(\widehat{\mathcal{F}}^n,s,a)<\beta\}
\end{equation}

\paragraph{Bonus Function}
In a more refined form, \textbf{LPO} devises bonus function in both the known and unknown sets.

\begin{equation}\label{E5}
	\begin{aligned}
	b^n(s,a)&=2b_w^n(s,a)+b_{\text{1}}^n(s,a),\  \text{where} \\
	b_w^n(s,a)&=\frac{1}{\beta}\ \omega(\widehat{\mathcal{F}}^n,s,a)\textbf{1}\{s\in\mathcal{K}^n\},\ \text{and}\ \\  b_{\text{1}}^n(s,a)&=\frac{3}{1-\gamma}\textbf{1}\{s\notin\mathcal{K}^n\}
	\end{aligned}
\end{equation}
On the unknown state-actions, the bonus is a constant $\frac{3}{1-\gamma}$, which is the largest value of the original reward over a trajectory. This will force the agent out of the known set and explore the unknown parts of the \textbf{MDP}. On the known state-actions, the uncertainty is measured by the width function.

Notice that our algorithm explore the environment in a much more sophisticated and efficient way than \citep{feng2021provably} does. Our algorithm \textbf{LPO} not only explores the unknown part using the indicator $b_{\text{1}}^n(s,a)$, but also takes the uncertainty information $b_w^n(s,a)$ in the known set into account. Consequently, \textbf{LPO} still explores the state-action pair in the known set until it is sufficiently understood. Moreover, since the size of sensitivity dataset is bounded by $O(d\log N)$, where $d$ is the eluder dimension, the average computing time of our bonus function is largely reduced.
\subsection{Inner Loop}
In the inner routine,  the \textbf{Policy Update} initializes the policy to be a uniform distribution and encourages the policy to explore the unknown state-actions. Next, we adopt the online learning algorithm to update the policy, which is an actor-critic pattern. This update rule is equivalent to Natural Policy Gradient (NPG) algorithm for log-linear policies \citep{kakade2001natural, agarwal2020optimality}, where we fit the critic with Monte Carlo samples and update the actor using exponential weights. As mentioned in \citep{agarwal2020optimality},  Monte Carlo sampling has an advantage of assuring better robustness to model misspecification, but produces huge source of sample complexity.

\paragraph{Sample efficient policy evaluation oracle via importance sampling.}
In the \textbf{Policy Update} routine, the policy obtained in each iteration needs to be evaluated. In a most direct way, the agent needs to interact with the environment by Monte Carlo sampling and estimate the $Q$-function for each policy, and this leads to the waste of samples. In order to improve the sample complexity of \textbf{Policy Update} while keeping the robustness property, we design a sample efficient policy evaluation oracle by applying trajectory-level importance sampling on past Monte Carlo return estimates \citep{precup2000eligibility}. To be specific, at iteration $\underline{k}$ in the inner loop, the agent will collect data by routine \textbf{Behaviour Policy Sampling} (\cref{A4}), and the dataset obtained in this iteration will be reused for the next $\kappa$ turns. At iteration $k$ ($k\leq \underline{k}+\kappa $), the \textbf{Policy Evaluation Oracle} (\cref{A5}) can estimate the Monte Carlo return for the current policy $\pi_k$ by reweighting the samples with importance sampling. With the reweighted random return, the oracle fits the critic via least square regression and outputs an estimated $\widehat{Q}_k$ for policy $\pi_k$. To this end, we update the policy by following the NPG rule.  Although the technique above can largely reduce the number of interactions with environment (from $K$ to $\lceil\frac{K}{\kappa}\rceil$), the selection of $\kappa$ greatly influences the variance of importance sampling estimator, which ultimately challenges the robustness property. In fact,  We need to set $\kappa = \widetilde{O}(\sqrt{K})$ in order to keep a stable variance of the estimator (see Lemma \ref{E.3} and Remark \ref{kappa} for details).

\begin{remark} 
If we have obtained a full coverage dataset over state-action space (e.g. using bonus-driven reward to collect data \citep{jin2020provably, wang2020reward}),  the policy evaluation oracle can evaluate all the policies by using the above mentioned dataset and only needs $\widetilde{O}(\frac{1}{\varepsilon^2})$ samples. However, this oracle depends on the ($\zeta$-approximate) linear MDP, which is a stronger assumption than that in our setting.
\end{remark}

\paragraph{Pessimistic critic to produce one-sided errors} From the line 9 in \cref{A5}, we  find that the critic fitting is actually the Monte Carlo return minus the initial bonus. Therefore, an intuitive form for the critic estimate is $\widehat{Q}_{b^n}^k(s,a)=f_k(s,a)+b^n(s,a)$. However, in line 10 of \cref{A5}, we only partially make up for the subtracted term and define the critic estimate as $\widehat{Q}_{b^n}^k(s,a)=f_k(s,a)+\frac{1}{2}b^n(s,a)$. This introduces a negative bias in the estimate. However, in the later proof we can see that $\widehat{Q}_{b^n}^k(s,a)$ is still optimistic relative to $Q^k(s,a)$. This one-sided error property plays an essential role of bounding the gap between $Q_{b^n}^{k,*}$ and $\widehat{Q}_{b^n}^k(s,a)$, which ultimately improves the sample complexity.

\section{Theoretical Guarantee} \label{sec 4}
In this section, we will provide our main result of \textbf{LPO} under some assumptions. The main theorem (\cref{thm:bigtheorem body}) is presented in this section and the complete proof is in the appendix.

First of all, the sample complexity of algorithms with function approximation depends on the complexity of the function class. To measure this complexity, we adopt the notion of eluder dimension which is first mentioned in \citep{russo2013eluder}.

 \begin{definition}
\label{def:inj}
(Eluder dimension). $\varepsilon \geq 0$ and $\mathcal{Z}=\left\{\left(s_{i}, a_{i}\right)\right\}_{i=1}^{n} \subseteq \mathcal{S} \times \mathcal{A}$ be a sequence of state-action pairs.
\begin{itemize}
\item{A state-action pair $(s, a) \in \mathcal{S} \times \mathcal{A}$ is $\varepsilon$-dependent on $\mathcal{Z}$ with respect to $\mathcal{F}$ if any $f, f^{\prime} \in \mathcal{F}$ satisfying $\left\|f-f^{\prime}\right\|_{\mathcal{Z}} \leq \varepsilon$ also satisfies $\left|f(s, a)-f^{\prime}(s, a)\right| \leq \varepsilon$.}
\item{An $(s, a)$ is $\varepsilon$-independent of $\mathcal{Z}$ with respect to $\mathcal{F}$ if $(s, a)$ is not $\varepsilon$-dependent on $\mathcal{Z}$.}
\item{The eluder dimension $\operatorname{dim}_{E}(\mathcal{F}, \varepsilon)$ of a function class $\mathcal{F}$ is the length of the longest sequence of elements in $\mathcal{S} \times \mathcal{A}$ such that, for some $\varepsilon^{\prime} \geq \varepsilon$, every element is $\varepsilon^{\prime}$-independent of its predecessors.}
\end{itemize}
\end{definition}


\begin{remark}
In fact, \citep{russo2013eluder} has shown several bounds for eluder dimension in some special cases. For example, when $\mathcal{F}$ is the class of linear functions, i.e. $f_{\theta}(s,a)=\theta^\top\phi(s,a)$ with a given feature $\phi: \mathcal{S}\times\mathcal{A} \rightarrow \mathbb{R}^d$, or the class of generalized linear functions of the form $f_{\theta}(s,a)=g(\theta^\top\phi(s,a))$ where $g$ is a differentiable and strictly increasing
function, $\operatorname{dim}_{E}(\mathcal{F}, \varepsilon)=O(d\log(1/\varepsilon))$. Recently, more general complexity measures for non-linear function class have been proposed in \citep{jin2021bellman, foster2021statistical, chen2022general}. However, the adoption of eluder dimension allows us to use computation-friendly optimization methods (e.g. dynamic programming, policy gradient) whereas theirs do not directly imply computationally tractable and implementable approaches. If only statistical complexities are considered, we believe our techniques could be  extended to their complexity measures.
\end{remark}

Next, we assume that the function class $\mathcal{F}$ and the state-actions $\mathcal{S}\times\mathcal{A}$ have bounded covering numbers.

\begin{assumption} \label{ass 4.2}
($\varepsilon$-cover). For any $\varepsilon>0$, the following holds:

1. there exists an $\varepsilon$-cover $\mathcal{C}(\mathcal{F}, \varepsilon) \subseteq \mathcal{F}$ with size $|\mathcal{C}(\mathcal{F}, \varepsilon)| \leq \mathcal{N}(\mathcal{F}, \varepsilon)$, such that for any $f \in \mathcal{F}$, there exists $f^{\prime} \in \mathcal{C}(\mathcal{F}, \varepsilon)$ with $\left\|f-f^{\prime}\right\|_{\infty} \leq \varepsilon$;

2. there exists an $\varepsilon$-cover $\mathcal{C}(\mathcal{S} \times \mathcal{A}, \varepsilon)$ with size $|\mathcal{C}(\mathcal{S} \times \mathcal{A}, \varepsilon)| \leq \mathcal{N}(\mathcal{S} \times \mathcal{A}, \varepsilon)$, such that for any $(s, a) \in \mathcal{S} \times \mathcal{A}$, there exists $\left(s^{\prime}, a^{\prime}\right) \in \mathcal{C}(\mathcal{S} \times \mathcal{A}, \varepsilon)$ with $\max _{f \in \mathcal{F}}\left|f(s, a)-f\left(s^{\prime}, a^{\prime}\right)\right| \leq \varepsilon$.
\end{assumption}

\begin{remark}
Assumption \ref{ass 4.2} is rather standard. Since our algorithm complexity depends only logarithmically on $\mathcal{N}(\mathcal{F},\cdot)$ and $\mathcal{N}(\mathcal{S} \times \mathcal{A},\cdot)$, it is even acceptable to have exponential size of these covering numbers.
\end{remark}

Next, we enforce a bound on the values of the functions class: 
\begin{assumption} \label{ass 4.3}
(Regularity). We assume that $\sup _{f \in \mathcal{F}}\|f\|_{\infty} \leq W$.
\end{assumption}
Assumption \ref{ass 4.3} is mild as nearly all the function classes used in practice have bounded magnitude in the input domain of interests. 
In general, one shall not expect an arbitrary function class could provide good guarantees in approximating a policy. In this section, we apply the following completeness condition to characterize whether the function class $\mathcal{F}$ fits to solve RL problems.
\begin{assumption} \label{ass 4.4}
$\left(\mathcal{F}\right.$-closedness). 
$$
\mathcal{T}^{\pi} f(s, a):=\mathbb{E}^{\pi}\left[r(s, a)+\gamma f\left(s^{\prime}, a^{\prime}\right) \mid s, a\right] .
$$

We assume that for all $\pi \in\{\mathcal{S} \rightarrow \Delta(\mathcal{A})\}$ and $g: \mathcal{S} \times \mathcal{A} \rightarrow\left[0, W\right]$, we have $\mathcal{T}^{\pi} g \in \mathcal{F}$.
\end{assumption}

\begin{remark}
Assumption \ref{ass 4.4} is a closedness assumption on $\mathcal{F}$, which enhances its representability to deal with critic fitting. For some special cases, like linear $f$ in the linear MDP \citep{yang2019sample, jin2020provably}, this assumption always holds. If this assumption does not hold, which means $Q$-function may not realizable in our function class $\mathcal{F}$, then there exists a $\epsilon_{\text{bias}}$ when we approximate the true $Q$-function. This model misspecified setting will be discussed in Assumption \ref{B.2}.
\end{remark}

With the above assumptions, we have the following sample complexity result for \textbf{LPO}.

\begin{theorem}
\label{thm:bigtheorem body} 
(Main Results: Sample Complexity of LPO). Under Assumptions \ref{ass 4.2}, \ref{ass 4.3}, and \ref{ass 4.4}, for any comparator $\widetilde{\pi}$, a fixed failure probability $\delta$, eluder dimension $d =\text{dim}(\mathcal{F},1/N)$, a suboptimality gap $\varepsilon$ and appropriate input hyperparameters: $N \geq \widetilde{O}\left(\frac{{d}^2}{(1-\gamma)^4\varepsilon^2}\right), K = \widetilde{O}\left(\frac{\ln |\mathcal{A}| W^{2}}{(1-\gamma)^{2} \varepsilon^{2}}\right), M \geq \widetilde{O}\left(\frac{{d}^2}{(1-\gamma)^4\varepsilon^2}\right), \eta = \widetilde{O}\left(\frac{\sqrt{\ln |\mathcal{A}|}}{\sqrt{K} W}\right), \kappa = \widetilde{O}\left(\frac{1-\gamma}{\eta W}\right)$, 
our algorithm returns a policy $\pi^{\text{LPO}}$, satisfying 
$$
\left(V^{\widetilde{\pi}}-V^{\pi^{\text{LPO}}}\right)\left(s_{0}\right) \leq \varepsilon.
$$
with probability at least $1-\delta$ after taking  at most $\widetilde{O}\left(\frac{{d}^3}{(1-\gamma)^{8} \varepsilon^{3}}\right)$ samples.
\end{theorem}

\begin{remark}
The complete proof of our main theorem is presented in \cref{M-thm}. For the model misspecified case, which means Assumption \ref{ass 4.4} does not hold, there exists a $\epsilon_{\text {bias }}$ in our regret bound (see details in Lemma \ref{B.10})

\end{remark}

\section{Practical Implementation and Experiments} \label{sec 5}
In this section, we introduce a practical approach to implementing our proposed theoretical algorithm and show our experiment results. 

\subsection{Practical Implementation of \textbf{LPO}}

\begin{figure*}[!htb]\label{mainexperiment}
\includegraphics[width=\linewidth]{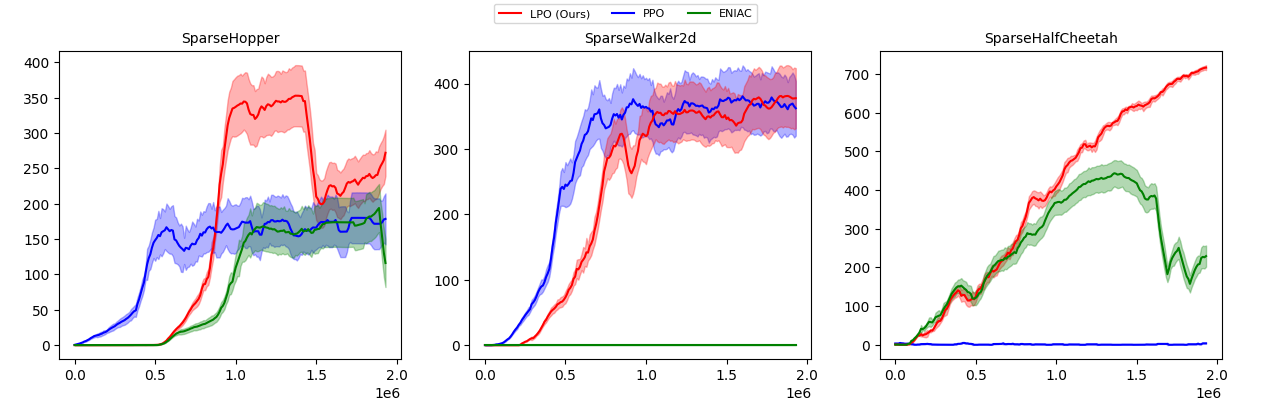}
\caption{Performance of \textbf{PPO}-Based algorithms on sparse-reward MuJoCo localmotion environments. Lines are evaluation results averaged over 5 random seeds every 10k steps, the shaded area represents the standard deviation.} \label{fig:1}
\end{figure*}

The width function in our theoretical framework enables our policy gradient algorithm to explore efficiently. The value of the width function should be large over novel state-action pairs and shrink over not novel. Intuitively, the width function over all state-action pairs should be its maximum at the beginning of learning and decrease along the way. To this end, we leverage the random network distillation technique proposed by \citep{burda2018exploration}. We randomly initialize two neural networks $f$ and $f'$ that maps from $\mathcal{A} \times \mathcal{S}$ to $\mathbb{R}^d$ with the same architecture (but different parameters). And our bonus $b(s, a)$ is defined as $b(s, a) = \Vert f(s, a) - f'(s, a) \Vert^2$. During training, we fix $f'$ and train $f$ to minimize $\Vert f(s, a) - f'(s, a) \Vert^2$ with gradient descent over state-action pairs that the agent has seen during the sampling procedure, i.e. distilling the randomly initialized network $f'$ into a trained one $f$. Over time, this residual-dependent bonus will decrease over state-action pairs that agents commonly visit.

With bonus calculated in the way described above, at each Monte Carlo session, we can calculate an intrinsic advantage estimation $\hat{A}^{(int)}$, which will affect our policy gradient along with the extrinsic advantage estimation $\hat{A}^{(ext)}$. The gradient of policy parametrized by $\phi$ is given by:

\begin{equation}\label{totaladvantage}
    \hat{A}^{(total)}_t = \alpha \hat{A}^{(ext)}_t + \beta \hat{A}^{(int)}_t
\end{equation}

where $\alpha$ and $\beta$ are hyperparameters that control how much we want to emphasize the importance of exploration, in our experiment, we use $\alpha = 2$ and $\beta = 1$. And the $\hat{A}_{ext}$, $\hat{A}_{int}$ are calculated with generalized advantage estimation \citep{schulman2015high}, and the estimation of advantages over a time period of $T$ is given by:

\begin{equation}\label{advcalculation}
    \begin{aligned}
        \hat{A}^{(ext)}_t & = \sum_{i=t}^T (\gamma^{(ext)} \lambda)^{i - t} [r_i + \gamma^{(ext)} V(s_{i+1}) - V(s_i)] \\
        \hat{A}^{(int)}_t & = \sum_{i=t}^T (\gamma^{(int)} \lambda)^{i - t} [b_i + \gamma^{(int)} V^{(int)}(s_{i+1}) \\ 
        & - V^{(int)}(s_i)] \\
    \end{aligned}
\end{equation}

where $V$ and $V^{(int)}$ are two value estimator parametrized that predicts extrinsic and intrinsic value separately. We train value functions to fit the Monte Carlo estimation of values of the current policy. 


In our theoretical framework, the sensitivity is computed using \eqref{E1} and only designed to achieve boundness on the final theoretical guarantee, but not for practical implementation. We overcome this issue by introducing a coarse, yet effective approximation of sensitivity sampling: gradually increasing the samples we collect for Monte Carlo estimation. For each sampling procedure at time $t$, we collect $max\{N, (1 + \frac{1}{T})^tN\}$ samples ($N$ is the number of samples we collect at the first sampling procedure). This simple mechanism makes the agent collect more and more samples for each training loop, which allow the agent to explore more freely at the early stage of learning, and forces the agent to explore more carefully at the late stage, as using more data for each training loop will shrink the value of width function in a more stable way. The algorithm is shown in Algorithm \ref{ea}.

\begin{algorithm}[t]
\caption{\textbf{LPO (Practical Implementation)}} \label{ea}
\begin{algorithmic}[1]
\STATE \textbf{Input}: Width function $(f, f')$, Policy $\pi_{\phi_0}$, Value networks $(V^{(ext)}$, $V^{(int)})$
\STATE \textbf{Hyperparameters}: $N, K, \gamma, \lambda,\alpha, \beta$
\STATE For all $s \in S$, initialize $\pi^{0}(\cdot|s)=\text{Unif}(\mathcal{A})$
\FOR{
    $k = 0,1,2,3,...,K$
}
    \STATE $T \gets \lceil (1 + \frac{1}{K})^k N \rceil$
    \STATE Run policy $\pi_{\phi}$ for $T$ steps $\rightarrow D_k$
    \STATE Calculate $A^{(ext)}$, $A^{(int)}$ using \ref{advcalculation} using $\lambda$
    \STATE Calculate $A^{(total)}$ using \ref{totaladvantage} using $\alpha$, $\beta$
    \STATE Optimize $\pi_{\phi}$, $(V^{(ext)}$, $V^{(int)})$ using \textbf{PPO} with $A^{(total)}$ as advantage estimation
    \STATE Optimize $f$ to fit $f'$ w.r.t. $D_k$
\ENDFOR
\STATE \textbf{Output}: Unif($\pi^0,\pi^1,\cdots,\pi^{N-1}$)
\end{algorithmic}
\end{algorithm}

\subsection{Experiments}
To further illustrate the effectiveness of our width function and our proposed sensitivity sampling, we compare \citep{schulman2017proximal, feng2021provably} with our proposed \textbf{LPO} in sparse reward MuJoCo environments \citep{6386109}. We re-implement \citep{feng2021provably} with the random network distillation method, as we find the original implementation of width function was not numerically stable. More details are in the discussion section. 

The sparse MuJoCo environment is different from the usual localmotion task, where rewards are dense and given according to the speed of the robots, in sparse environments, reward $(+1)$ is only given whenever the robot exceeds a certain speed, and no reward is given otherwise. Such environments are more difficult than their dense reward counterpart in the sense that the agent needs to actively explore the environment strategically and find a good policy. \textbf{PPO} manages to find rewards in SparseHopper and SparseWalker2d, but fails in SparseHalfCheetah. Although \textbf{ENIAC} \citep{feng2021provably} also uses intrinsic rewards for strategic exploration, it still fails in the SparseWalker2d environment. This might be due to the intrinsic reward of \textbf{ENIAC} switching too fast, thus the exploration is not persistent enough for the agent to find the reward. In contrast, our method succeeds in all three environments, the result is shown Figure \ref{fig:1}. 

\subsection{Limitation of Previous Implementations} 
Note that we do not compare our method directly with implementations in \citep{agarwal2020pc, feng2021provably}, as we discovered some limitations presented in their implementations.
We will discuss this in more details in Section~\ref{App F} about their limitations in terms of batch normalization and mis-implementations of core components in existing approaches. 
We illustrate the issue by directly running their published code. As shown in Figure \ref{fig:2}, our approaches and the other baseline approaches \citep{stable-baselines3} successfully solve the problem in a few epochs, while their implementations fail to achieve similar performance.

\begin{figure}[ht]
\vskip 0.2in
\begin{center}
\centerline{\includegraphics[width=0.5\columnwidth]{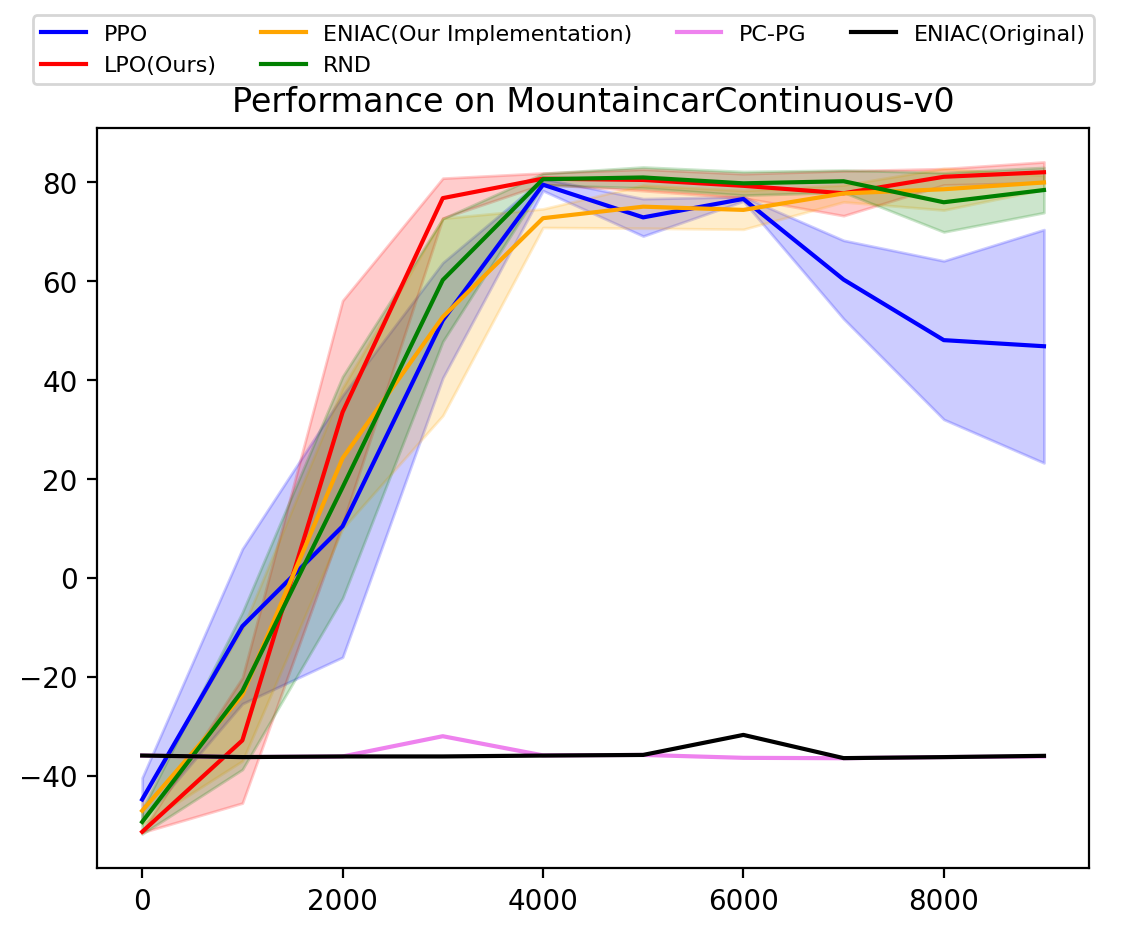}}
\caption{Performance comparison between the original implementations of \textbf{ENIAC, PC-PG} and our implementation of \textbf{ENIAC, LPO}. Lines are evaluation results averaged over 5 random seeds every 10k steps, the shaded area represents the standard deviation.}\label{fig:2}
\end{center}
\vskip -0.2in
\end{figure}

\

\section{Conclusion} \label{sec 6}
In this paper, we establish a low-switching sample-efficient policy optimization algorithm with general function approximation using online sensitivity sampling and data reuse techniques. Our algorithm largely improves the  sample complexity in \citep{feng2021provably}, while still keeping its robustness to model misspecification. Our numerical experiments also empirically prove the efficiency of the low-switching technique in policy gradient algorithms.

\section{Acknowledgements}

 This work was supported in part by DARPA under agreement HR00112190130, NSF grant 2221871, and an Amazon Research Grant. The authors would also like to thank Dingwen Kong for useful discussions.

	\bibliography{Reference}
	\bibliographystyle{icml2023}

	\newpage
	\appendix
	\onecolumn

\begin{algorithm}[tb]
\caption{\textbf{Behaviour Policy Sampling}} \label{A4}
\begin{algorithmic}[1]
\STATE \textbf{Input:} Behaviour Policy $\pi$, Policy Cover $\pi^{1:n}$
\STATE $D=\emptyset$
\FOR{$i=1,\cdots,M$} 
	\STATE Sample j uniformly at random in $1:n$
	\STATE $(s,a)\leftarrow \textbf{d-sampler}(\pi_j,\nu)$
	\STATE Sample $h\geq1$ with probability $\gamma^{h-1}(1-\gamma)$
	\STATE Continue the rollout from $(s,a)$ by executing $\pi$ for $h-1$ steps
	\STATE Storing the rollout $D[i]\leftarrow \{(s_1,a_1,\cdots,s_h,a_h)\}$ where $(s_1,a_1)=(s,a)$
\ENDFOR
\STATE \textbf{return:} $D$
\end{algorithmic}
\end{algorithm}

\begin{algorithm}[t]
\caption{\textbf{Policy Evaluation Oracle}} \label{A5}
\begin{algorithmic}[1]
\STATE \textbf{Input}: Evaluate policy $\pi$, Dataset $D$, Behaviour policy $\underline{\pi}$,  Bonus function $b$
\FOR{$i=1,2,\cdots,M$}
	\STATE $\lambda_i\leftarrow\prod\limits_{\tau=2}^{|D[i]|}\frac{\pi(a_{\tau}|s_{\tau})}{\underline{\pi}(a_{\tau}|s_{\tau})}$
        \STATE $(s_i^F,a_i^F)\leftarrow D[i]$ 's first sample
        \STATE $(s_i^L,a_i^L)\leftarrow D[i]$ 's last sample
        \STATE $G_i \leftarrow \frac{1}{1-\gamma}[r(s_i^L,a_i^L)+b(s_i^L,a_i^L)]$
\ENDFOR
\STATE \textbf{end for}
\STATE $\widehat{f}\leftarrow \underset{f \in \mathcal{F}}{\text{argmin}}\sum_{i=1}^{M}\big{(}f(s_i^F,a_i^F)-(\lambda_i G_i-b(s_i^F,a_i^F))\big{)}^2$  \label{L9} 
\STATE \textbf{return:} $\widehat{Q}(s,a)=\widehat{f}(s,a)+\frac{1}{2}b(s,a),\ \forall s \in \mathcal{K}^n\ \text\ {and}\ \widehat{Q}(s,a)=\widehat{f}(s,a)+b(s,a),\ $otherwise  \label{L10}
\end{algorithmic}
\end{algorithm}

\begin{algorithm}[t]
\caption{\textbf{d-sampler}} \label{A6}
\begin{algorithmic}[1]
\STATE \textbf{Input}: $\nu\in\Delta(S\times A),\pi$.
\STATE Sample $s_0,a_0\sim\nu$.
\STATE Sample $\tau\geq1$ with probability $\gamma^{\tau-1}(1-\gamma)$.
\STATE Execute $\pi$ for $\tau-1$ steps, giving state s.
\STATE Sample action $a\sim\pi(\cdot|s)$.
\STATE \textbf{return} $(s,a)$.
\end{algorithmic}
\end{algorithm}

\clearpage

\section{Remaining Algorithm Pseudocodes}
We provide the remaining algorithms including Behaviour Policy Sampling (\cref{A4}), Policy Evaluation Oracle (\cref{A5}), and the visitation distribution sampler (\cref{A6}).

\section{Main Analysis} \label{App B}

In this section, we provide the main guarantees for \textbf{LPO}. 

\subsection{Proof Setup}
\paragraph{Bonus and auxiliary MDP}
 To begin with, we introduce the concept of optimisic (bonus-added) and auxiliary MDP, which is also mentioned in \citep{agarwal2020pc,feng2021provably}. However, we design these MDPs with very different bonus functions.

For each epoch $n\in[N]$, we consider an arbitrary fixed comparator policy $\widetilde{\pi}$ (e.g., an optimal policy) and three MDPs: the original MDP $\mathcal{M}:=(\mathcal{S},\mathcal{A},P,r,\gamma)$, the bonus-added MDP $\mathcal{M}_{b^n}:=(\mathcal{S},\mathcal{A},P,r+b^n,\gamma)$, and an auxiliary MDP $\mathcal{M}^n:=(\mathcal{S},\mathcal{A}\cup\{{a^\dagger}\},P^n,r^n,\gamma)$, where $a^\dagger$ is an extra action which is only available for $s\notin \mathcal{K}^n$. For all $(s,a)\in\mathcal{S}\times\mathcal{A}$,
$$
P^n(\cdot|s,a)=P(\cdot|s,a),\ r^n(s,a)=r(s,a)+b^n(s,a).
$$
For $s\notin\mathcal{K}^n$

$$
P^n(s|s,a^\dagger)=1, r^n(s,a^\dagger)=3
$$
The above equation actually exhibits its property to absorb and provide maximum rewards for agent outside the known states. For readers' convenience, we present the definition of bonus function and known states.

The bonus function $b^n:\mathcal{S}\times\mathcal{A}\rightarrow\mathbb{R}$ defined as 
\begin{equation}
	\begin{aligned}
	b_w^n(s,a)&=\frac{1}{\beta}\ \omega(\widehat{\mathcal{F}}^n,s,a)\textbf{1}\{s\in\mathcal{K}^n\}\\b_{\text{1}}^n(s,a)&=\frac{3}{1-\gamma}\textbf{1}\{s\notin\mathcal{K}^n\}\\b^n(s,a)&=2b_w^n(s,a)+b_{\text{1}}^n(s,a)
	\end{aligned}
\end{equation}
Known states
\begin{equation}
	\begin{aligned}
	\mathcal{K}^n=&\{s\in\mathcal{S}|\ \forall a\in\mathcal{A},\omega(\widehat{\mathcal{F}}^n,s,a)<\beta\}\\\mathcal{K}^n=&\{(s,a)\in\mathcal{S}\times\mathcal{A}|\ \omega(\widehat{\mathcal{F}}^n,s,a)<\beta\}
	\end{aligned}
\end{equation}
\paragraph{}
Given the auxiliary MDP $\mathcal{M}^n$, we define $\widetilde{\pi}^n(\cdot|s)=\widetilde{\pi}(\cdot|s)\  \text{for} \ s\in \mathcal{K}^n \ \text{and} \ \widetilde{\pi}^n(a^\dagger|s)=1\ \text{for}\ s\notin\mathcal{K}^n$. We denote by $\widetilde{d}_{\mathcal{M}^n}$ the state-action distribution induced by $\widetilde{\pi}^n$ on $\mathcal{M}^n$ and $d^{\widetilde{\pi}}$ the state-action distribution induced by $\widetilde{\pi}$ on $\mathcal{M}$.
\paragraph{}
Given a policy $\pi$, we denote by $V_{b^n}^{\pi},Q_{b^n}^{\pi},\text{and}\  A_{b^n}^{\pi}$ the state-value, Q-value, and the advantage function of $\pi$ on $\mathcal{M}_{b^n}$, and $V_{\mathcal{M}^n}^{\pi}, Q_{\mathcal{M}^n}^{\pi}$,and $A_{\mathcal{M}^n}^{\pi}$ for the counterparts on $\mathcal{M}^n$, and we define $Q^{\pi}(s,\widetilde{\pi}):=\mathbb{E}_{a\sim\widetilde{\pi}(\cdot|s)}[Q^{\pi}(s,a)]$. \newline

Based on the above definitions and notations, we have the following lemmas.

\begin{lemma} \label{B.5}
(Distribution Dominance) \citep{feng2021provably}. Consider any state $s \in \mathcal{K}^{n}$, we have:

$$
\tilde{d}_{\mathcal{M}^{n}}(s, a) \leq d^{\tilde{\pi}}(s, a), \quad \forall a \in \mathcal{A} .
$$
\end{lemma}
\begin{proof}
We prove by induction over the time steps along the horizon $h$. We denote the state-action distribution at the $h_{\mathrm{th}}$ step following $\tilde{\pi}^{n}$ on $\mathcal{M}^{n}$ as $\tilde{d}_{\mathcal{M}^{n}, h}$.

Starting at $h=0$, if $s_{0} \in \mathcal{K}^{n}$, then $\tilde{\pi}^{n}\left(\cdot \mid s_{0}\right)=\tilde{\pi}\left(\cdot \mid s_{0}\right)$ and

$$
\tilde{d}_{\mathcal{M}^{n}, 0}\left(s_{0}, a\right)=d_{0}^{\tilde{\pi}}\left(s_{0}, a\right), \quad \forall a \in \mathcal{A} .
$$

Now we assume that at step $h$, for all $s \in \mathcal{K}^{n}$, it holds that

$$
\tilde{d}_{\mathcal{M}^{n}, h}(s, a) \leq d_{h}^{\tilde{n}}(s, a), \forall a \in \mathcal{A}
$$

Then, for step $h+1$, by definition we have that for $s \in \mathcal{K}^{n}$

$$
\begin{aligned}
\tilde{d}_{\mathcal{M}^{n}, h+1}(s) &=\sum_{s^{\prime}, a^{\prime}} \tilde{d}_{\mathcal{M}^{n}, h}\left(s^{\prime}, a^{\prime}\right) P_{\mathcal{M}^{n}}\left(s \mid s^{\prime}, a^{\prime}\right) \\
&=\sum_{s^{\prime}, a^{\prime}} 1\left\{s^{\prime} \in \mathcal{K}^{n}\right\} \tilde{d}_{\mathcal{M}^{n}, h}\left(s^{\prime}, a^{\prime}\right) P_{\mathcal{M}^{n}}\left(s \mid s^{\prime}, a^{\prime}\right) \\
&=\sum_{s^{\prime}, a^{\prime}} 1\left\{s^{\prime} \in \mathcal{K}^{n}\right\} \tilde{d}_{\mathcal{M}^{n}, h}\left(s^{\prime}, a^{\prime}\right) P\left(s \mid s^{\prime}, a^{\prime}\right)
\end{aligned}
$$

where the second line is due to that if $s^{\prime} \notin \mathcal{K}^{n}, \tilde{\pi}$ will deterministically pick $a^{\dagger}$ and $P_{\mathcal{M}^{n}}\left(s \mid s^{\prime}, a^{\dagger}\right)=0 .$ On the other hand, for $d_{h+1}^{\tilde{\pi}}(s, a)$, it holds that for $s \in \mathcal{K}^{n}$,

$$
\begin{aligned}
d_{h+1}^{\tilde{\pi}}(s) &=\sum_{s^{\prime}, a^{\prime}} d_{h}^{\tilde{\pi}}\left(s^{\prime}, a^{\prime}\right) P\left(s \mid s^{\prime}, a^{\prime}\right) \\
&=\sum_{s^{\prime}, a^{\prime}} 1\left\{s^{\prime} \in \mathcal{K}^{n}\right\} d_{h}^{\tilde{\pi}}\left(s^{\prime}, a^{\prime}\right) P\left(s \mid s^{\prime}, a^{\prime}\right)+\sum_{s^{\prime}, a^{\prime}} 1\left\{s^{\prime} \notin \mathcal{K}^{n}\right\} d_{h}^{\tilde{\pi}}\left(s^{\prime}, a^{\prime}\right) P\left(s \mid s^{\prime}, a^{\prime}\right) \\
& \geq \sum_{s^{\prime}, a^{\prime}} 1\left\{s^{\prime} \in \mathcal{K}^{n}\right\} d_{h}^{\tilde{\pi}}\left(s^{\prime}, a^{\prime}\right) P\left(s \mid s^{\prime}, a^{\prime}\right) \\
& \geq \sum_{s^{\prime}, a^{\prime}} 1\left\{s^{\prime} \in \mathcal{K}^{n}\right\} \tilde{d}_{\mathcal{M}^{n}, h}\left(s^{\prime}, a^{\prime}\right) P\left(s \mid s^{\prime}, a^{\prime}\right) \\
&=\tilde{d}_{\mathcal{M}^{n}, h+1}(s) .
\end{aligned}
$$

Using the fact that $\tilde{\pi}^{n}(\cdot \mid s)=\tilde{\pi}(\cdot \mid s)$ for $s \in \mathcal{K}^{n}$, we conclude that the inductive hypothesis holds at $h+1$ as well. Using the definition of the average state-action distribution, we conclude the proof.

\end{proof}

\begin{lemma} \label{B.8}
(Partial optimism) \citep{zanette2021cautiously}. Fix a policy $\widetilde{\pi}$ that never takes $a^\dagger$. In any episode $n$ it holds that
$$
V_{\mathcal{M}^n}^{\widetilde{\pi}^n}(\widetilde{s})-V^{\widetilde{\pi}}(\widetilde{s})\geq\frac{1}{1-\gamma}\mathbb{E}_{s\sim d_{\widetilde{s}}^{\widetilde{\pi}}}2b_\omega^n(s,\widetilde{\pi})
$$
\end{lemma}
\begin{proof}

Notice that if $s\notin\mathcal{K}^n$, then $V_{\mathcal{M}^n}^{\widetilde{\pi}^n}(s)=\frac{3}{1-\gamma}$ as the policy self-loops in $s$ by taking $a^\dagger$ there. Otherwise, 
\begin{equation}
	\begin{aligned}
V_{\mathcal{M}^n}^{\widetilde{\pi}^n}(s)&=\mathbb{E}_{a\sim\widetilde{\pi}^n(\cdot|s)}Q_{\mathcal{M}^n}^{\widetilde{\pi}^n}(s,a)\\ &= \frac{1}{1-\gamma}\mathbb{E}_{a\sim\tilde{\pi}^n(\cdot|s)}\mathbb{E}_{(s',a')\sim\widetilde{d}_{\mathcal{M}^n}|(s,a)}r^n(s',a')\\ &\leq \frac{3}{1-\gamma}
	\end{aligned}
\end{equation}
Therefore, $V_{\mathcal{M}^n}^{\widetilde{\pi}^n}(s)\leq\frac{3}{1-\gamma}$. Using the performance difference lemma in \citep{kakade2001natural}, we get:
\begin{equation}
	\begin{aligned}
	&(1-\gamma)(V_{\mathcal{M}^n}^{\widetilde{\pi}^n}(\widetilde{s})-V_{\mathcal{M}^n}^{\widetilde{\pi}}(\widetilde{s}))=\mathbb{E}_{(s,a)\sim d_{\widetilde{s}}^{\widetilde{\pi}}}[-A_{\mathcal{M}^n}^{\widetilde{\pi}^n}(s,a)]\\&=\mathbb{E}_{(s,a)\sim d_{\widetilde{s}}^{\widetilde{\pi}}}[V_{\mathcal{M}^n}^{\widetilde{\pi}^n}(s)-Q_{\mathcal{M}^n}^{\widetilde{\pi}^n}(s,a)]\\&=\mathbb{E}_{s\sim d_{\widetilde{s}}^{\widetilde{\pi}}}[Q_{\mathcal{M}^n}^{\widetilde{\pi}^n}(s,\widetilde{\pi}^n)-Q_{\mathcal{M}^n}^{\widetilde{\pi}^n}(s,\widetilde{\pi})]\\&=\mathbb{E}_{s\sim d_{\widetilde{s}}^{\widetilde{\pi}}}[\left(Q_{\mathcal{M}^n}^{\widetilde{\pi}^n}(s,\widetilde{\pi}^n)-Q_{\mathcal{M}^n}^{\widetilde{\pi}^n}(s,\widetilde{\pi})\right)\textbf{1}\{s\notin\mathcal{K}^n\}]\\&=\mathbb{E}_{s\sim d_{\widetilde{s}}^{\widetilde{\pi}}}[\left(\frac{3}{1-\gamma}-r(s,\widetilde{\pi})-2b_{\omega}^n(s,\widetilde{\pi})-b_{\text{1}}^n(s,\widetilde{\pi})-\gamma\mathbb{E}_{s'\sim P(s,\widetilde{\pi})}V_{\mathcal{M}^n}^{\widetilde{\pi}^n}(s')\right)\textbf{1}\{s\notin\mathcal{K}^n\}]
	\end{aligned}
\end{equation}
Since $r(s,\widetilde{\pi})+2b_{\omega}^n(s,\widetilde{\pi})+\gamma\mathbb{E}_{s'\sim P(s,\widetilde{\pi})}V_{\mathcal{M}^n}^{\widetilde{\pi}^n}(s')\leq 1+2+\frac{3\gamma}{1-\gamma} \leq\frac{3}{1-\gamma}$\newline

Thus,
\begin{equation}
	\begin{aligned}
	V_{\mathcal{M}^n}^{\widetilde{\pi}^n}(\widetilde{s})&\geq V_{\mathcal{M}^n}^{\widetilde{\pi}}(\widetilde{s})-\frac{1}{1-\gamma}\mathbb{E}_{s\sim d_{\widetilde{s}}^{\widetilde{\pi}}}b_1^n(s,\widetilde{\pi})\\&=V^{\widetilde{\pi}}(\widetilde{s})+\frac{1}{1-\gamma}\mathbb{E}_{s\sim d_{\widetilde{s}}^{\widetilde{\pi}}}b^n(s,\widetilde{\pi})-\frac{1}{1-\gamma}\mathbb{E}_{s\sim d_{\widetilde{s}}^{\widetilde{\pi}}}b_1^n(s,\widetilde{\pi})\\&=V^{\widetilde{\pi}}(\widetilde{s})+\frac{1}{1-\gamma}\mathbb{E}_{s\sim d_{\widetilde{s}}^{\widetilde{\pi}}}2b_{\omega}^n(s,\widetilde{\pi})
	\end{aligned}
\end{equation}
\end{proof}

\begin{lemma} \label{B.9}
(Negative Advantage) \citep{zanette2021cautiously}. 
$$
A^{\pi^n}_{\mathcal{M}^n}(s,\widetilde{\pi}^n)\textbf{1}\{s\notin\mathcal{K}^n\}\leq 0
$$
\end{lemma}
\begin{proof}
Assume $s\notin \mathcal{K}^n$, then $Q^{\pi^n}_{\mathcal{M}^n}(s,\widetilde{\pi}^n)=3+\gamma V_{\mathcal{M}^n}^{{\pi}^n}(s)$. Note that if $s\notin \mathcal{K}^n$, $\pi^n$ takes an action $a\not= a^{\dagger}$ such that $b_1^n(s,a)=\frac{3}{1-\gamma}$. Therefore, $V_{\mathcal{M}^n}^{{\pi}^n}(s)\geq\frac{3}{1-\gamma}$.\newline Combining the two expressions we obtain that, in any state $s\notin \mathcal{K}^n$, 
$$
A^{\pi^n}_{\mathcal{M}^n}(s,\widetilde{\pi}^n)=Q^{\pi^n}_{\mathcal{M}^n}(s,\widetilde{\pi}^n)- V_{\mathcal{M}^n}^{{\pi}^n}(s)=3+\gamma V_{\mathcal{M}^n}^{{\pi}^n}(s)- V_{\mathcal{M}^n}^{{\pi}^n}(s)\leq 0
$$
\end{proof}

\subsection{Proof Sketch}
In order to bound the gap of values between the output policy $\pi^{\text{LPO}}=$Unif($\pi^0,\pi^1,\cdots,\pi^{N-1}$) and the comparator $\widetilde{\pi}$, we need to analyze the gap between $V^{\widetilde{\pi}}$ and $V^{\pi^n}$ for each $n \in [N]$. With the above three lemmas based on the structure of the well-designed MDPs, we are able to obtain the following regret decomposition (see details in Lemma \ref{B.10} (Regret decomposition)):
\begin{equation}
    \begin{aligned}
        \left(V^{\widetilde{\pi}}-V^{{\pi}^n}\right)(s_0)& \leq \frac{1}{1-\gamma}\left[\underbrace{\mathop{\text{sup}}\limits_{s\in\mathcal{K}^n}\widehat{A}^{\pi^n}_{\mathcal{M}^n}(s,\widetilde{\pi})\textbf{1}\{s\in\mathcal{K}^n\}}_{\text{term 1}}+\underbrace{\mathbb{E}_{s\sim \widetilde{d}_{\mathcal{M}^n}}[A^{\pi^n}_{\mathcal{M}^n}(s,\widetilde{\pi})-A_{\mathcal{M}^n}^*(s,\widetilde{\pi})]\textbf{1}\{s\in\mathcal{K}^n\}}_{\text{term 2}}\right.\\&\ \ \ \ \ \ \ \ \ \ \ \ +\left.\underbrace{\mathbb{E}_{s\sim \widetilde{d}_{\mathcal{M}^n}}[A^*_{\mathcal{M}^n}(s,\widetilde{\pi})-\widehat{A}_{\mathcal{M}^n}^{{\pi}^n}(s,\widetilde{\pi})]\textbf{1}\{s\in\mathcal{K}^n\}}_{\text{term 3}}-\underbrace{\mathbb{E}_{s\sim d^{\widetilde{\pi}}|s_0}2b_{\omega}^n(s,\widetilde{\pi})}_{\text{term 4}}+\underbrace{\mathbb{E}_{s\sim d^n|s_0}b^n(s,\pi^n)}_{\text{term 5}}\right]        
    \end{aligned}
\end{equation}

Now we discuss the details of each term.
\begin{itemize}
    \item term 1 represents the \emph{Solver Error}, which measures the performance of policy $\pi^n$ in terms of empirical advantage function on known states. This term can be bounded by Lemma \ref{B.7} (NPG Analysis).
    \item term 2 represents the \emph{Approximation Error}, which exists when our function class $\mathcal{F}$ do not have enough representability to deal with critic fitting, and this term can be controlled with Assumption \ref{B.2} (Bounded Transfer Error) and Lemma \ref{B.6}.
    \item term 3 represents the \emph{Statistical Error}, which is the average error between the empirical and optimal advantage function under known states. This term can be bounded by term 4 (the expectation of width function) according to Lemma \ref{B.5} and Lemma \ref{E.6}.
    \item term 5 is the expectation of bonus function under policy $\pi^n$, and the bound of bonuses is achieved in Lemma \ref{D.2}.
\end{itemize}

\subsection{Analysis of LPO}
In this part, we follow the above steps of proof to establish the result of our main theorem.

First, we introduce some notions and assumptions to handle the model misspecification. These notions have been discussed in \citep{agarwal2020pc, feng2021provably}.

\begin{definition} \label{B.1}
(Transfer error).  
Given a target function $g:\mathcal{S}\times\mathcal{A}\rightarrow\mathbb{R}$, we define the critic loss function $L(f;d;g)$ with $d\in\Delta(\mathcal{S}\times\mathcal{A})$ as:
$$
L(f;d;g):=\mathbb{E}_{(s,a)\sim d }\left[(f(s,a)-g(s,a))^2\right]
$$
\end{definition}
We let $Q_{b^n}^{{\pi}^n}$, $Q_{b^n}^{{\pi}_k}$ to be the $Q$-function in the bonus-added MDP for a given outer iteration $n$ and an inner iteration $k$. Then the transfer error for a fixed comparator $\widetilde{\pi}$ is defined as 
$$
\epsilon_k^n:=\inf _{f \in \mathcal{F}_{k}^{n}} L\left(f ; \tilde{d}, Q_{b^{n}}^{\pi_k}-b^{n}\right),
$$
where $\mathcal{F}_{k}^{n}:=\operatorname{argmin}_{f \in \mathcal{F}} L\left(f ; \rho_{c o v}^{n}, Q_{b^{n}}^{\pi_k}-b^{n}\right)$ and $\tilde{d}(s, a):=$ $d_{s_{0}}^{\tilde{\pi}}(s) \circ \operatorname{Unif}(\mathcal{A})$.

\begin{assumption} \label{B.2}
(Bounded Transfer Error). 
For the fixed comparator policy $\tilde{\pi}$ , for every epoch $n \in[N]$ and every iteration $k$ inside epoch $n$, we assume that there exists a constant $\epsilon_{\text {bias }}$ such that

$$
\epsilon_k^n \leq \epsilon_{\text {bias }},
$$
\end{assumption}

Notice that $\epsilon_{\text {bias }}$ measures both approximation error and distribution shift error. By the definition of transfer error, we can select a function $\tilde{f}_{k}^{n} \in \mathcal{F}_{k}^{n}$, such that
$$
L\left(\tilde{f}_{k}^{n}; \tilde{d};  Q_{b^{n}}^{\pi_k}-b^{n}\right) \leq 2 \epsilon_{\text {bias }}
$$

\begin{assumption} \label{B.3}
For the same loss $L$ in the Definition \ref{B.1} and the fitter $\tilde{f}_{k}^{n}$ in Assumption \ref{B.2}, we assume that there exists some $C \geq 1$ and $\epsilon_{0} \geq 0$ such that for any $f \in \mathcal{F}$,

$$
\mathbb{E}_{(s, a) \sim \rho_{c o v}^{n}}\left[\left(f(s, a)-\tilde{f}_{k}^{n}(s, a)\right)^{2}\right] \leq C \cdot\left(L\left(f ; \rho_{c o v}^{n}, Q_{b^{n}}^{\pi_k}-b^{n}\right)-L\left(\tilde{f}_{k}^{n} ; \rho_{c o v}^{n}, Q_{b^{n}}^{\pi_k}-b^{n}\right)\right)+\epsilon_{0}
$$

for $n \in[N]$ and $0 \leq k \leq K-1$.
\end{assumption}

\begin{remark} \label{remark app 1}
Under Assumption \ref{ass 4.4}, for every $n \in [N]$ and $k \in [K]$, $Q_{b^{n}}^{\pi_k}(s,a)-b^{n}(s,a)=\mathbb{E}^{\pi_{k}^{n}}\left[r(s, a)+\gamma Q_{b^{n}}^{\pi_k}\left(s^{\prime}, a^{\prime}\right)|s,a\right] \in \mathcal{F}$. Thus, $\epsilon_{\text {bias }}$ can take value 0 and $\tilde{f}_{k}^{n}=Q_{b^{n}}^{\pi_k}-b^{n}$. Further in Assumption \ref{B.3}, we have

$$
\mathbb{E}_{(s, a) \sim \rho_{c o v}^{n}}\left[\left(f(s, a)-\tilde{f}_{k}^{n}(s, a)\right)^{2}\right]=L\left(f ; \rho_{c o v}^{n}, Q_{b^{n}}^{\pi_k}-b^{n}\right) .
$$

Hence, $C$ can take value 1 and $\epsilon_{0}=0$. If $Q_{b^{n}}^{\pi_k}-b^{n}$ is not realizable in $\mathcal{F}, \epsilon_{\text {bias }}$ and $\epsilon_{0}$ could be strictly positive. Hence, the above two assumptions are generalized version of the closedness condition considering model misspecification. Next, we define the optimal regression fit considering the loss function and its corresponding advantage functions.
\end{remark}

\begin{definition} \label{B.4}
\begin{equation}
	\begin{aligned}
	&f^{n,*}\in\mathop{\text{argmin}}\limits_{f\in\mathcal{F}}L(f;\rho^n,Q_{b^n}^{{\pi}^n}-b^n),\ f_k^*\in\mathop{\text{argmin}}\limits_{f\in\mathcal{F}}L(f;\rho^n,Q_{b^n}^{{\pi}_k}-b^n)\\&A_{b^n}^*(s,a)=f^{n,*}(s,a)+b^n(s,a)-\mathbb{E}_{a'\sim\pi^n(\cdot|s)}\left[f^{n,*}(s,a')+b^n(s,a')\right]\\&A_{b^n}^{k,*}(s,a)=f_k^*(s,a)+b^n(s,a)-\mathbb{E}_{a'\sim\pi_k(\cdot|s)}\left[f_k^*(s,a')+b^n(s,a')\right]\\
	\end{aligned}
\end{equation}
In the later proof, we select $f^{n,*}$, $f_k^*$ not only to  be the optimal solution with respect to the above loss function, but also satisfy the inequality in Assumption \ref{B.3}, just like $\tilde{f}_{k}^{n}$.
\end{definition}

\begin{lemma} \label{B.6}
(Advantage Transfer Error decomposition). 
We have that

$$
\mathbb{E}_{(s, a) \sim \tilde{d}_{\mathcal{M}^{n}}}\left(A_{b^{n}}^{k}-A_{b^{n}}^{k, *}\right) \mathbf{1}\left\{s \in \mathcal{K}^{n}\right\}<4 \sqrt{|\mathcal{A}| \epsilon_{\text {bias }}} .
$$
\end{lemma}

\begin{proof}

We have

$$
\begin{aligned}
&\mathbb{E}_{(s, a) \sim \tilde{d}_{\mathcal{M}^{n}}}\left(A_{b^{n}}^{k}-A_{b^{n}}^{k, *}\right) 1\left\{s \in \mathcal{K}^{n}\right\} \\
&=\mathbb{E}_{(s, a) \sim \tilde{d}_{\mathcal{M}^{n}}}\left(Q_{b^{n}}^{k}-f_{k}^{*}-b^{n}\right) \mathbf{1}\left\{s \in \mathcal{K}^{n}\right\}-\mathbb{E}_{s \sim \tilde{d}_{\mathcal{M}^{n}}, a \sim \pi_{k}(\cdot \mid s)}\left(Q_{b^{n}}^{k}-f_{k}^{*}-b^{n}\right) 1\left\{s \in \mathcal{K}^{n}\right\} \\
&\leq \sqrt{\mathbb{E}_{(s, a) \sim \tilde{d}_{\mathcal{M}^{n}}}\left(Q_{b^{n}}^{k}-f_{k}^{*}-b^{n}\right)^{2} \mathbf{1}\left\{s \in \mathcal{K}^{n}\right\}}+\sqrt{\mathbb{E}_{s \sim \tilde{d}_{\mathcal{M}^{n}, a \sim \pi_{k}(\mid s)}}\left(Q_{b^{n}}^{k}-f_{k}^{*}-b^{n}\right)^{2} 1\left\{s \in \mathcal{K}^{n}\right\}} \\
&\leq \sqrt{\mathbb{E}_{(s, a) \sim d^{\bar{\pi}}}\left(Q_{b^{n}}^{k}-f_{k}^{*}-b^{n}\right)^{2} \mathbf{1}\left\{s \in \mathcal{K}^{n}\right\}}+\sqrt{\mathbb{E}_{s \sim d^{\bar{\pi}}, a \sim \pi_{k}(\cdot \mid s)}\left(Q_{b^{n}}^{k}-f_{k}^{*}-b^{n}\right)^{2} \mathbf{1}\left\{s \in \mathcal{K}^{n}\right\}} \\
&=\sqrt{\mathbb{E}_{(s, a) \sim \tilde{d}}|\mathcal{A}| \tilde{\pi}(a \mid s) \cdot\left(Q_{b^{n}}^{k}-f_{k}^{*}-b^{n}\right)^{2} \mathbf{1}\left\{s \in \mathcal{K}^{n}\right\}}+\sqrt{\mathbb{E}_{(s, a) \sim \tilde{d}}|\mathcal{A}| \pi_{k}(a \mid s) \cdot\left(Q_{b^{n}}^{k}-f_{k}^{*}-b^{n}\right)^{2} \mathbf{1}\left\{s \in \mathcal{K}^{n}\right\}} \\
&<4 \sqrt{|\mathcal{A}| \epsilon_{\mathrm{bias}}},
\end{aligned}
$$

where the first inequality is by Cauchy-Schwarz, the second inequality is by distribution dominance, and the last two lines follow the bounded transfer error assumption and the definition of $f_{k}^{*}$.
\end{proof}

Next, we provide the NPG Analysis.
\begin{lemma} \label{B.7}
(NPG Analysis) \citep{agarwal2020pc}.

Take stepsize $\eta=$ $\sqrt{\frac{\log (|\mathcal{A}|)}{16 W^{2} K}}$ in the algorithm, then for any $n \in[N]$ we have

$$
\sum_{k=0}^{K-1} \mathbb{E}_{(s, a) \sim \tilde{d}_{\mathcal{M}^{n}}}\left[\widehat{A}_{b^{n}}^{k}(s, a) \mathbf{1}\left\{s \in \mathcal{K}^{n}\right\}\right] \leq 8 W \sqrt{\log (|\mathcal{A}|) K}
$$
\end{lemma}

\begin{proof}

Here we omit the index $n$ for simplicity. From the NPG update rule

$$
\begin{aligned}
\pi_{k+1}(\cdot \mid s) & \propto \pi_{k}(\cdot \mid s) e^{\eta \widehat{Q}_{k}(s, \cdot)} \\
& \propto \pi_{k}(\cdot \mid s) e^{\eta \widehat{Q}_{k}(s, \cdot)} e^{-\eta \widehat{V}_{k}(s)} \\
&=\pi_{k}(\cdot \mid s) e^{\eta \widehat{A}_{k}(s, \cdot)}
\end{aligned}
$$

if we define the normalizer

$$
z_{k}(s)=\sum_{a^{\prime}} \pi_{k}\left(a^{\prime} \mid s\right) e^{\eta \widehat{A}_{k}\left(s, a^{\prime}\right)}
$$

then the update can be written as

$$
\pi_{k+1}(\cdot \mid s)=\frac{\pi_{k}(\cdot \mid s) e^{\eta \widehat{A}_{k}(s, \cdot)}}{z_{k}(s)}
$$

Then for any $s \in \mathcal{K}^{n}$,

$$
\begin{aligned}
&\mathbf{K L}\left(\tilde{\pi}(\cdot \mid s), \pi_{k+1}(\cdot \mid s)\right)-\mathbf{K L}\left(\widetilde{\pi}(\cdot \mid s), \pi_{k}(\cdot \mid s)\right) \\
&=\sum_{a} \tilde{\pi}(a \mid s) \log \frac{\tilde{\pi}(a \mid s)}{\pi_{k+1}(a \mid s)}-\sum_{a} \tilde{\pi}(a \mid s) \log \frac{\tilde{\pi}(a \mid s)}{\pi_{k}(a \mid s)} \\
&=\sum_{a} \tilde{\pi}(a \mid s) \log \frac{\pi_{k}(a \mid s)}{\pi_{k+1}(a \mid s)} \\
&=\sum_{a} \tilde{\pi}(a \mid s) \log \left(z_{k} e^{-\eta \widehat{A}_{k}(s, a)}\right) \\
&=-\eta \sum_{a} \widetilde{\pi}(a \mid s) \widehat{A}_{k}(s, a)+\log z_{k}(s)
\end{aligned}
$$

Since $\left|\widehat{A}_{k}(s, a)\right| \leq 4 W$ and when $T>\log (|\mathcal{A}|), \eta<1 /(4 W)$, we have $\eta \widehat{A}_{k}(s, a) \leq 1 .$ By the inequality that $\exp (x) \leq 1+x+x^{2}$ for $x \leq 1$ and $\log (1+x) \leq x$ for $x>-1$

$$
\log \left(z_{k}(s)\right) \leq \eta \sum_{a} \pi_{k}(a \mid s) \widehat{A}_{k}(s, a)+16 \eta^{2} W^{2}=16 \eta^{2} W^{2}
$$

Thus for $s \in \mathcal{K}^{n}$ we have

$$
\mathbf{K L}\left(\tilde{\pi}(\cdot \mid s), \pi_{k+1}(\cdot \mid s)\right)-\mathbf{K L}\left(\tilde{\pi}(\cdot \mid s), \pi_{k}(\cdot \mid s)\right) \leq-\eta \mathbb{E}_{a \sim \tilde{\pi}^{n}(\cdot \mid s)}\left[\widehat{A}_{k}(s, a)\right]+16 \eta^{2} W^{2}
$$

By taking sum over $k$, we get

$$
\begin{aligned}
& \sum_{k=0}^{K-1} \mathbb{E}_{(s, a) \sim \tilde{d}_{\mathcal{M}^{n}}}\left[\widehat{A}_{k}(s, a) \mathbf{1}\left\{s \in \mathcal{K}^{n}\right\}\right] \\
=& \sum_{k=0}^{K-1} \frac{1}{\eta} \mathbb{E}_{s \sim \tilde{d}_{\mathcal{M}^{n}}}\left[\left(\mathbf{K L}\left(\tilde{\pi}(\cdot \mid s), \pi_{0}(\cdot \mid s)\right)-\mathbf{K L}\left(\tilde{\pi}(\cdot \mid s), \pi_{K}(\cdot \mid s)\right)\right) \mathbf{1}\left\{s \in \mathcal{K}^{n}\right\}\right]+16 \eta K W^{2} \\
\leq & \log (|\mathcal{A}|) / \eta+16 \eta K W^{2} \\
\leq & 8 W \sqrt{\log (|\mathcal{A}|) K} .
\end{aligned}
$$
\end{proof}

Now we are ready to analyze the regret decomposition.

\begin{lemma} \label{B.10}
(Regret decomposition).
With probability at least $1-\delta$ it holds that 
\begin{equation}
	\begin{aligned}
	\frac{1}{N}\sum_{n=1}^N\left(V^{\widetilde{\pi}}-V^{{\pi}^n}\right)(s_0)\leq \frac{\mathcal{R}(K)}{(1-\gamma)K}+
\frac{2\sqrt{2A\epsilon_{\text{bias}}}}{1-\gamma}+\frac{1}{\sqrt{N}}\widetilde{O}\left(\frac{\sqrt{{d}^2\epsilon}}{(1-\gamma)^2\beta}\right)
	\end{aligned}
\end{equation}
\end{lemma}

\begin{proof} 
 Fix a policy $\widetilde{\pi}$ on $\mathcal{M}$. Consider the following decomposition for an outer episode $n$,
\begin{equation}
	\begin{aligned}
	\left(V^{\widetilde{\pi}}-V^{{\pi}^n}\right)(s_0)&=\underbrace{V^{\widetilde{\pi}}(s_0)+\frac{1}{1-\gamma}\mathbb{E}_{s\sim d^{\widetilde{\pi}}|s_0}2b_{\omega}^n(s,\widetilde{\pi})}_{\leq V_{\mathcal{M}^n}^{\tilde{\pi}^n}(s_0) \ \text{by Lemma \ref{B.8}}}-\underbrace{V^{{\pi}^n}(s_0)-\frac{1}{1-\gamma}\mathbb{E}_{s\sim d^n|s_0}b^n(s,\pi^n)}_{=-V_{b^n}^{\pi^n}}\\&+\frac{1}{1-\gamma}\underbrace{\left[-\mathbb{E}_{s\sim d^{\widetilde{\pi}}|s_0}2b_{\omega}^n(s,\widetilde{\pi})+\mathbb{E}_{s\sim d^n|s_0}b^n(s,\pi^n)\right]}_{\overset{def}{=}B^n}
	\end{aligned}
\end{equation}
We put the term $B^n$ aside for a moment and use performance difference lemma to obtain
\begin{equation}
	\begin{aligned}
	V_{\mathcal{M}^n}^{\tilde{\pi}^n}(s_0)-V_{b^n}^{\pi^n}(s_0)&=V_{\mathcal{M}^n}^{\tilde{\pi}^n}(s_0)-V_{\mathcal{M}^n}^{\pi^n}(s_0)\\&=\frac{1}{1-\gamma}\mathbb{E}_{s\sim \widetilde{d}_{\mathcal{M}^n}}\left[A^{\pi^n}_{\mathcal{M}^n}(s,\widetilde{\pi}^n)\right]\\&=\frac{1}{1-\gamma}\mathbb{E}_{s\sim \widetilde{d}_{\mathcal{M}^n}}\left[A^{\pi^n}_{\mathcal{M}^n}(s,\widetilde{\pi}^n)\textbf{1}\{s\in\mathcal{K}^n\}+\underbrace{A^{\pi^n}_{\mathcal{M}^n}(s,\widetilde{\pi}^n)\textbf{1}\{s\notin\mathcal{K}^n\}}_{\leq0 \ \text{by Lemma \ref{B.9}}}\right]\\&\leq \frac{1}{1-\gamma}\mathbb{E}_{s\sim \widetilde{d}_{\mathcal{M}^n}}\left[A^{\pi^n}_{\mathcal{M}^n}(s,\widetilde{\pi})\textbf{1}\{s\in\mathcal{K}^n\}\right]
	\end{aligned}
\end{equation}
where the last step is because on states $s\in\mathcal{K}^n$ we have $\widetilde{\pi}^n(\cdot|s)=\widetilde{\pi}(\cdot|s)$. \newline
Define
\begin{equation}
	\begin{aligned}
	\widehat{A}_{b^n}^k(s,a)&=\widehat{Q}_{b^n}^k(s,a)-\widehat{V}_{b^n}^k(s)\\&=f_k(s,a)+b_{\omega}^n(s,a)-\mathbb{E}_{a'\sim\pi_k(\cdot|s)}\left[f_k(s,a')+b_{\omega}^n(s,a')\right]
	\end{aligned}
\end{equation}
and
\begin{equation}
	\begin{aligned}
	\widehat{A}_{b^n}^{\pi^n}(s,a)=\frac{1}{K}\sum_{k=0}^{K-1}\widehat{A}_{b^n}^k(s,a)
	\end{aligned}
\end{equation}
Then we get
\begin{equation}
	\begin{aligned}
	&=\frac{1}{1-\gamma}\left[\mathbb{E}_{s\sim \widetilde{d}_{\mathcal{M}^n}}\widehat{A}^{\pi^n}_{\mathcal{M}^n}(s,\widetilde{\pi})\textbf{1}\{s\in\mathcal{K}^n\}+\mathbb{E}_{s\sim \widetilde{d}_{\mathcal{M}^n}}\left[A^{\pi^n}_{\mathcal{M}^n}(s,\widetilde{\pi})-\widehat{A}^{\pi^n}_{\mathcal{M}^n}(s,\widetilde{\pi})\right]\textbf{1}\{s\in\mathcal{K}^n\}\right]\\ &\leq\frac{1}{1-\gamma}\left[\underbrace{\mathop{\text{sup}}\limits_{s\in\mathcal{K}^n}\widehat{A}^{\pi^n}_{\mathcal{M}^n}(s,\widetilde{\pi})\textbf{1}\{s\in\mathcal{K}^n\}}_{\text{term 1}}+\underbrace{\mathbb{E}_{s\sim \widetilde{d}_{\mathcal{M}^n}}[A^{\pi^n}_{\mathcal{M}^n}(s,\widetilde{\pi})-A_{\mathcal{M}^n}^*(s,\widetilde{\pi})]\textbf{1}\{s\in\mathcal{K}^n\}}_{\text{term 2}}\right.\\&\ \ \ \ \ \ \ \ \ \ \ \ +\left.\underbrace{\mathbb{E}_{s\sim \widetilde{d}_{\mathcal{M}^n}}[A^*_{\mathcal{M}^n}(s,\widetilde{\pi})-\widehat{A}_{\mathcal{M}^n}^{{\pi}^n}(s,\widetilde{\pi})]\textbf{1}\{s\in\mathcal{K}^n\}}_{\text{term 3}}\right]
	\end{aligned}
\end{equation}
The first term can be bounded by Lemma \ref{B.7} (NPG lemma):
\begin{equation}
	\begin{aligned}
	\mathop{\text{sup}}\limits_{s\in\mathcal{K}^n}\widehat{A}^{\pi^n}_{\mathcal{M}^n}(s,\widetilde{\pi})\textbf{1}\{s\in\mathcal{K}^n\}= \mathop{\text{sup}}\limits_{s\in\mathcal{K}^n}\frac{1}{K}\sum_{k=0}^{K-1}\mathbb{E}_{a\sim\tilde{\pi}(\cdot|s)}\widehat{A}_{b^n}^k(s,a)\textbf{1}\{s\in\mathcal{K}^n\}\leq\frac{\mathcal{R}(K)}{K}
	\end{aligned}
\end{equation}
The second term can be bounded by Lemma \ref{B.5}, \ref{B.6}
\begin{equation}
	\begin{aligned}
	&\mathbb{E}_{s\sim \widetilde{d}_{\mathcal{M}^n}}[A^{\pi^n}_{\mathcal{M}^n}(s,\widetilde{\pi})-A_{\mathcal{M}^n}^*(s,\widetilde{\pi})]\textbf{1}\{s\in\mathcal{K}^n\}\\ &\leq \mathbb{E}_{s\sim d^{\tilde{\pi}}}[A^{\pi^n}_{\mathcal{M}^n}(s,\widetilde{\pi})-A_{\mathcal{M}^n}^*(s,\widetilde{\pi})]\textbf{1}\{s\in\mathcal{K}^n\}\\ &\leq 2\sqrt{2A\epsilon_{\text{bias}}}
	\end{aligned}
\end{equation}
The third term can be bounded by Lemma \ref{E.6}, which ensures that with probability at least $1-\frac{\delta}{2}$ it holds that
\begin{equation}
	\begin{aligned}
	\forall n\in [N],\ \forall k\in\{0,\cdots,K-1\},\ \forall (s,a)\in\mathcal{K}^n: \ 0\leq Q_{b^n}^{k,*}(s,a)-\widehat{Q}_{b^n}^{k}(s,a)\leq2b_{\omega}^n(s,a)
	\end{aligned}
\end{equation}
Then we have: $\forall n\in[N],\forall(s,a)\in\mathcal{K}^n$:
\begin{equation}
	\begin{aligned}
	A^*_{\mathcal{M}^n}(s,a)-\widehat{A}_{\mathcal{M}^n}^{{\pi}^n}(s,a)&=\frac{1}{K}\sum_{k=0}^{K-1}\left[\left( Q_{b^n}^{k,*}(s,a)-\widehat{Q}_{b^n}^{k}(s,a)\right)-\underbrace{\left(Q_{b^n}^{k,*}(s,\pi_k^n)-\widehat{Q}_{b^n}^{k}(s,\pi_k^n)\right)}_{\geq0}\right]\\ &\leq Q^*_{\mathcal{M}^n}(s,a)-\widehat{Q}_{\mathcal{M}^n}^{{\pi}^n}(s,a)
	\end{aligned}
\end{equation}
Apply the Lemma \ref{B.5}, we have
\begin{equation}
	\begin{aligned}
	\mathbb{E}_{s\sim \widetilde{d}_{\mathcal{M}^n}}[A^*_{\mathcal{M}^n}(s,\widetilde{\pi})-\widehat{A}_{\mathcal{M}^n}^{{\pi}^n}(s,\widetilde{\pi})]\textbf{1}\{s\in\mathcal{K}^n\}&\leq\mathbb{E}_{s\sim \widetilde{d}_{\mathcal{M}^n}}[Q^*_{\mathcal{M}^n}(s,\widetilde{\pi})-\widehat{Q}_{\mathcal{M}^n}^{{\pi}^n}(s,\widetilde{\pi})]\textbf{1}\{s\in\mathcal{K}^n\}\\ &\leq \mathbb{E}_{s\sim d^{\widetilde{\pi}}}[Q^*_{\mathcal{M}^n}(s,\widetilde{\pi})-\widehat{Q}_{\mathcal{M}^n}^{{\pi}^n}(s,\widetilde{\pi})]\textbf{1}\{s\in\mathcal{K}^n\}
	\end{aligned}
\end{equation}
As a result,
\begin{equation}
	\begin{aligned}
	\left(V^{\widetilde{\pi}}-V^{\pi^n}\right)(s_0)&\leq\frac{1}{1-\gamma}\left[\frac{\mathcal{R}(K)}{K}+ 2\sqrt{2A\epsilon_{\text{bias}}}+\mathbb{E}_{s\sim d^{\widetilde{\pi}}}2b_{\omega}^n(s,\widetilde{\pi})\textbf{1}\{s\in\mathcal{K}^n\}+B^n\right]\\&=\frac{1}{1-\gamma}\left[\frac{\mathcal{R}(K)}{K}+ 2\sqrt{2A\epsilon_{\text{bias}}}+\mathbb{E}_{s\sim d^n}b^n(s,\pi^n)\right]
	\end{aligned}
\end{equation}
And finally using the concentration of bonuses (Lemma \ref{D.2}) we get
\begin{equation}
	\begin{aligned}
	\frac{1}{N}\sum_{n=1}^N\left(V^{\widetilde{\pi}}-V^{\pi^n}\right)(s_0)&\leq\frac{\mathcal{R}(K)}{(1-\gamma)K}+
\frac{2\sqrt{2A\epsilon_{\text{bias}}}}{1-\gamma}+\frac{1}{N(1-\gamma)}\sum_{n=1}^N\mathbb{E}_{s\sim d^n}b^n(s,\pi^n)\\ &\leq \frac{\mathcal{R}(K)}{(1-\gamma)K}+
\frac{2\sqrt{2A\epsilon_{\text{bias}}}}{1-\gamma}+\frac{1}{\sqrt{N}}\widetilde{O}\left(\frac{\sqrt{{d}^2\epsilon}}{(1-\gamma)^2\beta}\right)
	\end{aligned}
\end{equation}
\end{proof}

Combining all previous lemmas, we have the following theorem about the sample complexity of our LPO.

\begin{theorem}
\label{thm:bigtheorem} \label{M-thm}
(Main Results: Sample Complexity of LPO). Under Assumptions \ref{ass 4.2}, \ref{ass 4.3}, and \ref{ass 4.4}, for any comparator $\widetilde{\pi}$, a fixed failure probability $\delta$, eluder dimension $d = \text{dim}(\mathcal{F},1/N)$, a suboptimality gap $\varepsilon$ and appropriate input hyperparameters: $N \geq \widetilde{O}\left(\frac{{d}^2}{(1-\gamma)^4\varepsilon^2}\right), K = \widetilde{O}\left(\frac{\ln |\mathcal{A}| W^{2}}{(1-\gamma)^{2} \varepsilon^{2}}\right), M \geq \widetilde{O}\left(\frac{{d}^2}{(1-\gamma)^4\varepsilon^2}\right), \eta = \widetilde{O}\left(\frac{\sqrt{\ln |\mathcal{A}|}}{\sqrt{K} W}\right), \kappa = \widetilde{O}\left(\frac{1-\gamma}{\eta W}\right)$, 
our algorithm returns a policy $\pi^{\text{LPO}}$, satisfying 
$$
\left(V^{\widetilde{\pi}}-V^{\pi^{\text{LPO}}}\right)\left(s_{0}\right) \leq \varepsilon.
$$
with probability at least $1-\delta$ after taking  at most $\widetilde{O}\left(\frac{{d}^3}{(1-\gamma)^{8} \varepsilon^{3}}\right)$ samples.
\end{theorem}
\begin{proof}
First, let's consider Lemma \ref{B.10} (Regret decomposition).
We need ensure

$$
\frac{\mathcal{R}(K)}{(1-\gamma) K}=\frac{8 W}{(1-\gamma)} \sqrt{\frac{\ln |\mathcal{A}|}{K}} \leq \frac{\varepsilon}{2} \quad \longrightarrow \quad K=\widetilde{O}\left(\frac{\ln |\mathcal{A}| W^{2}}{(1-\gamma)^{2} \varepsilon^{2}}\right)
$$
This gives the inner iteration complexity. Next, $\beta$ can be any constant between 0 and 1, and recall that $\epsilon$ is the parameter that controls the width function \eqref{E3}. We set $\epsilon$ in the following form (see Lemma \ref{E.4} for justification)
$$
\epsilon=100\left(\frac{3}{2}C_1N\cdot\epsilon_{\text{stat}}+20NW\epsilon_1+\frac{1}{2}C_2\cdot\ln\left(\frac{N\mathcal{N}(\Delta\mathcal{F},2\epsilon_1)}{\delta'}\right)\right)
$$
and
$$
\epsilon_{\text{stat}}=\frac{500 C \cdot W^{4} \cdot \log \left(\frac{\mathcal{N}\left(\mathcal{F}, \epsilon_{2}\right)}{\delta_3}\right)}{M}+13 W^{2} \cdot \epsilon_{2}
$$
where $\epsilon_1, \epsilon_2$ represents the function cover radius. Since our algorithm complexity depends only logarithmically on the covering numbers, we can set the cover radius with any polynomial degree of $\varepsilon$. 
In fact, $\epsilon_1=O(\varepsilon^3)$, $\epsilon_2=O(\varepsilon^3)$, $\epsilon=O(\log N)$, 
$$
\begin{aligned}
\frac{1}{\sqrt{N}}\widetilde{O}\left(\frac{\sqrt{{d}^2\epsilon}}{(1-\gamma)^2\beta}\right)\leq\frac{\varepsilon}{2}\longrightarrow M=N\geq\widetilde{O}\left(\frac{{d}^2}{(1-\gamma)^4\varepsilon^2}\right)
\end{aligned}
$$

gives the outer iteration complexity and the number of samples collected by a single Monte Carlo trajectory.

Under Assumption \ref{ass 4.4}, which means $Q$-function is realizable in our function class $\mathcal{F}$, $\epsilon_{\text{bias}}=0$ (see Remark \ref{remark app 1} for justification).
After setting the hyperparameters above, with probability at least $1-\delta$, we have 
$$
\frac{1}{N}\sum_{n=1}^N\left(V^{\widetilde{\pi}}-V^{\pi^n}\right)(s_0)\leq \varepsilon
$$
Remember that our algorithm outputs a uniform mixture of policy cover $\pi^{\text{LPO}}=$Unif($\pi^0,\pi^1,\cdots,\pi^{N-1}$), so we have

$$
\left(V^{\widetilde{\pi}}-V^{\pi^{\text{LPO}}}\right)\left(s_{0}\right) \leq \varepsilon.
$$
Next, we consider the total samples we need to collect through steps of the algorithm.

Every time the bonus switches, \cref{A3} is invoked, and runs for $K$ iterations. From Lemma \ref{E.3} we know that once data are collected, they can be reused for the next $\kappa$ policies. Therefore, we actually run \cref{A4} for $\left\lceil\frac{K}{\kappa}\right\rceil$ times, and whenever invoking \cref{A4}, we need $M$ samples by Monte Carlo sampling. We denote $S$ to be the number of bonus switches given in Proposition \ref{C.7} (Number of Switches).

In total, the sample complexity of our algorithm is  
$$
\begin{aligned}
&\underbrace{S}\limits_{\text{number of inner loops invoked}}\ \ \times\ \  \underbrace{\left\lceil\frac{K}{\kappa}\right\rceil}\limits_{\text{the inner iteration}}\ \  \times\ \  \underbrace{M}\limits_{\text{Monte Carlo trajectories}} \\
 &=\widetilde{O}\left(d\times\frac{2 \ln (1 / \delta)\left(\frac{\sqrt{\ln |\mathcal{A}|}}{\sqrt{K} W}\right)(B+W)}{(1-\gamma) \ln 2} \times K\times M\right) \\&=\widetilde{O}\left(d\left(\frac{B}{W}+1\right) \frac{\sqrt{K}}{1-\gamma}M\right) \\
&=\widetilde{O}\left(\frac{d}{1-\gamma} \times \frac{W}{(1-\gamma) \varepsilon}\times \frac{{d}^2}{(1-\gamma)^4\varepsilon^2}\right) \\
&=O\left(\frac{{d}^3}{(1-\gamma)^{8} \varepsilon^3}\right)
\end{aligned}
$$

We complete the proof of our main theorem.
\end{proof}

\section{The Number of Switches} \label{App C}
In this section, we will give the bound of the number of switching policies in the outer loop. 

Recall that the width function is
$$
\omega({\widehat{\mathcal{F}}}^n,s,a) = \mathop{\text{sup}}\limits_{f_1,f_2\in\mathcal{F},||f_1-f_2||_{\widehat{\mathcal{Z}}^n}^2\leq\epsilon}|f_1(s,a)-f_2(s,a)|
$$
The parameter $\epsilon$ will be defined later in \eqref{eq 34}. In fact, we will show that $\epsilon= {O}(\log{N})$ in Lemma \ref{E.4} and \ref{E.5}.
First, we need to show that for every $n \in [N]$, the sensitivity dataset ${\widehat{\mathcal{Z}}}^n$ approximates the original dataset ${\mathcal{Z}}^n$. Our approach is inspired by \citep{kong2021online}.

For all $n \in[N]$ and $\alpha \in[\epsilon,+\infty)$, we define the following quantities

$$
\begin{aligned}
\underline{\mathcal{B}}^{n}(\alpha) &:=\left\{\left(f_{1}, f_{2}\right) \in \mathcal{F} \times \mathcal{F} \mid\left\|f_{1}-f_{2}\right\|_{\mathcal{Z}^{n}}^{2} \leq \alpha / 100\right\} \\
\mathcal{B}^{n}(\alpha) &:=\left\{\left(f_{1}, f_{2}\right) \in \mathcal{F} \times \mathcal{F} \mid \min \left\{\left\|f_{1}-f_{2}\right\|_{\widehat{\mathcal{Z}}^{n}}^{2}, 4 N W^{2}\right\} \leq \alpha\right\} \\
\overline{\mathcal{B}}^{n}(\alpha) &:=\left\{\left(f_{1}, f_{2}\right) \in \mathcal{F} \times \mathcal{F} \mid\left\|f_{1}-f_{2}\right\|_{\mathcal{Z}^{n}}^{2} \leq 100 \alpha\right\}
\end{aligned}
$$

For each $n \in[N]$, we use $\mathcal{E}^{n}(\alpha)$ to denote the event that

$$
\underline{\mathcal{B}}^{n}(\alpha) \subseteq \mathcal{B}^{n}(\alpha) \subseteq \overline{\mathcal{B}}^{n}(\alpha)
$$

Furthermore, we denote that

$$
\mathcal{E}^{n}:=\bigcap_{j=0}^{\infty} \mathcal{E}^{n}\left(100^{j} \epsilon\right),
$$

Our goal is to show that the event $\mathcal{E}^{n}$ holds with great probability, which means ${\widehat{\mathcal{Z}}}^n$ is a good approximation to ${\mathcal{Z}}^n$.
\\

Before presenting the proof, we need the following concentration inequality proved in \citep{freedman1975tail}.

\begin{lemma} \label{C.1}
Consider a real-valued martingale $\left\{Y_{k}: k=0,1,2, \cdots\right\}$ with difference sequence $\left\{X_{k}: k=0,1,2, \cdots\right\}$. Assume that the difference sequence is uniformly bounded:

$$
\left|X_{k}\right| \leq R \quad \text { almost surely for } \quad k=1,2,3, \cdots
$$
For a fixed $n \in \mathbb{N}$, assume that

$$
\sum_{k=1}^{n} \mathbb{E}_{k-1}\left(X_{k}^{2}\right) \leq \sigma^{2}
$$

almost surely. Then for all $t \geq 0$,

$$
P\left\{\left|Y_{n}-Y_{0}\right| \geq t\right\} \leq 2 \exp \left\{-\frac{t^{2} / 2}{\sigma^{2}+R t / 3}\right\}
$$
\end{lemma}

Furthermore, we need a bound on the number of elements in the sensitivity dataset. This is established in \citep{kong2021online}.

\begin{lemma} \label{C.2}
With probability at least $1-\delta / 64 N$,
$$
\left|\widehat{\mathcal{Z}}^{n}\right| \leq 64 N^{3} / \delta \quad \forall n \in[N] .
$$
\end{lemma}

The following lemma shows that if $\mathcal{E}^{n}$ happens, ${\widehat{\mathcal{Z}}}^n$ is a good approximation to ${\mathcal{Z}}^n$. And this is proved in \citep{kong2021online}.

\begin{lemma} \label{C.3}
If $\mathcal{E}^{n}$ occurs, then
$$
\frac{1}{10000}\left\|f_{1}-f_{2}\right\|_{\mathcal{Z}^{n}}^{2} \leq \min \left\{\left\|f_{1}-f_{2}\right\|_{\widehat{\mathcal{Z}}^{n}}^{2}, 4 N W^{2}\right\} \leq 10000\left\|f_{1}-f_{2}\right\|_{\mathcal{Z}^{n}}^{2}, \quad \forall\left\|f_{1}-f_{2}\right\|_{\mathcal{Z}^{n}}^{2}>100 \epsilon
$$

and

$$
\min \left\{\left\|f_{1}-f_{2}\right\|_{\widehat{\mathcal{Z}}^{n}}^{2}, 4 N W^{2}\right\} \leq 10000 \epsilon, \quad \forall\left\|f_{1}-f_{2}\right\|_{\mathcal{Z}^{n}}^{2} \leq 100 \epsilon
$$
\end{lemma}

To establish our result, we need the following lemma. The proof follows the approach of \citep{kong2021online}. We present it here for completeness.

\begin{lemma} \label{C.4}
For $\alpha \in\left[\epsilon, 4 N W^{2}\right]$

$$
\operatorname{Pr}\left(\mathcal{E}^{1} \mathcal{E}^{2} \ldots \mathcal{E}^{n-1}\left(\mathcal{E}^{n}(\alpha)\right)^{c}\right) \leq \delta /\left(32 N^{2}\right)
$$
\end{lemma}
\begin{proof}
We use $\overline{\mathcal{Z}}^{n}$ to denote the dataset without rounding, i.e., we replace every element $\hat{z}$ with $z$. Denote $C_{1}=C \cdot \log \left(N \cdot \mathcal{N}\left(\mathcal{F}, \sqrt{\delta / 64 N^{3}}\right) / \delta\right)$, where $C$ will be chosen appropriately later. We consider any fixed pair $\left(f_{1}, f_{2}\right) \in \mathcal{C}\left(\mathcal{F}, \sqrt{\delta /\left(64 N^{3}\right)}\right) \times \mathcal{C}\left(\mathcal{F}, \sqrt{\delta /\left(64 N^{3}\right)}\right)$.

For each $i \geq 2$, define

$$
Z_{i}=\max \left\{\left\|f_{1}-f_{2}\right\|_{\mathcal{Z}^{i}}^{2}, \min \left\{\left\|f_{1}-f_{2}\right\|_{\widehat{Z}^{i-1}}^{2}, 4 N W^{2}\right\}\right\}
$$

and

$$
Y_{i}= \begin{cases}\frac{1}{p_{z_{i-1}}}\left(f_{1}\left(z_{i-1}\right)-f_{2}\left(z_{i-1}\right)\right)^{2} & z_{i-1} \text { is added into } \overline{\mathcal{Z}}^{i} \text { and } Z_{i} \leq 2000000 \alpha \\ 0 & z_{i-1} \text { is not added into } \overline{\mathcal{Z}}^{i} \text { and } Z_{i} \leq 2000000 \alpha \\ \left(f_{1}\left(z_{i-1}\right)-f_{2}\left(z_{i-1}\right)\right)^{2} & Z_{i}>2000000 \alpha\end{cases}
$$

Note that $Z_{i}$ is constant under $\mathcal{F}_{i-1}$ and $Y_{i}$ is adapted to the filtration $\mathcal{F}_{i}$, thus

$$
\mathbb{E}_{i-1}\left[Y_{i}\right]=\left(f_{1}\left(z_{i-1}\right)-f_{2}\left(z_{i-1}\right)\right)^{2}
$$

now we bound $Y_{i}$ and its variance in order to use Freedman's inequality.

If $p_{z_{i-1}}=1$ or $Z_i>2000000 \alpha$, then $Y_{i}-\mathbb{E}_{i-1}\left[Y_{i}\right]=\operatorname{Var}_{i-1}\left[Y_{i}-\mathbb{E}_{i-1}\left[Y_{i}\right]\right]=$ 0 . Otherwise by the definition of $p_{z}$ we have

$$
\begin{aligned}
\left|Y_{i}-\mathbb{E}_{i-1}\left[Y_{i}\right]\right| & \leq\left(\min \left\{\left\|f_{1}-f_{2}\right\|_{\widehat{\mathcal{Z}}^{i-1}}^{2}, 4 N W^{2}\right\}+1\right) / C_{1} \\
& \leq 3000000 \alpha / C_{1}
\end{aligned}
$$

and

$$
\begin{aligned}
\operatorname{Var}_{i-1}\left[Y_{i}-\mathbb{E}_{i-1}\left[Y_{i}\right]\right] & \leq \frac{1}{p_{z_{i-1}}}\left(f_{1}\left(z_{i-1}\right)-f_{2}\left(z_{i-1}\right)\right)^{4} \\
& \leq\left(f_{1}\left(z_{i-1}\right)-f_{2}\left(z_{i-1}\right)\right)^{2} \cdot 3000000 \alpha / C_{1}
\end{aligned}
$$

Taking sum with respect to $i$ yields

$$
\sum_{i=2}^{n} \operatorname{Var}_{i-1}\left[Y_{i}-\mathbb{E}_{i-1}\left[Y_{i}\right]\right] \leq(3000000 \alpha)^{2} / C_{1}
$$

By Freedman's inequality, we have

$$
\begin{aligned}
&\mathbb{P}\left\{\left|\sum_{i=2}^{n}\left(Y_{i}-\mathbb{E}_{i-1}\left[Y_{i}\right]\right)\right| \geq \alpha / 100\right\} \\
&\leq 2 \exp \left\{-\frac{(\alpha / 100)^{2} / 2}{(3000000 \alpha)^{2} / C_{1}+\alpha \cdot 3000000 \alpha / 3 C_{1}}\right\} \\
&\leq\left(\delta / 64 N^{2}\right) /\left(\mathcal{N}\left(\mathcal{F}, \sqrt{\delta /\left(64 N^{3}\right)}\right)\right)^{2}
\end{aligned}
$$

where the last inequality is guaranteed by taking $C$ appropriately large.

Taking a union bound over all $\left(f_{1}, f_{2}\right) \in \mathcal{C}\left(\mathcal{F}, \sqrt{\delta /\left(64 N^{3}\right)}\right) \times \mathcal{C}\left(\mathcal{F}, \sqrt{\delta /\left(64 N^{3}\right)}\right)$, with probability at least $1-\delta /\left(64 T^{2}\right)$, we have

$$
\left|\sum_{i=2}^{n}\left(Y_{i}-\mathbb{E}_{i-1}\left[Y_{i}\right]\right)\right| \leq \alpha / 100 .
$$

In the rest of the proof, we condition on the event above and the event defined in Lemma \ref{C.2}.

\paragraph{Part 1} $\left(\underline{\mathcal{B}}^{n}(\alpha) \subseteq \mathcal{B}^{n}(\alpha)\right)$ : Consider any pair $f_{1}, f_{2} \in \mathcal{F}$ with $\left\|f_{1}-f_{2}\right\|_{\mathcal{Z}^{n}}^{2} \leq \alpha / 100$. From the definition we know that there exist $\left(\hat{f}_{1}, \hat{f}_{2}\right) \in \mathcal{C}\left(\mathcal{F}, \sqrt{\delta /\left(64 N^{3}\right)}\right) \times \mathcal{C}\left(\mathcal{F}, \sqrt{\delta /\left(64 N^{3}\right)}\right)$ such that $\left\|\hat{f}_{1}-f_{1}\right\|_{\infty},\left\|\hat{f}_{2}-f_{2}\right\|_{\infty} \leq \sqrt{\delta /\left(64 N^{3}\right)}$. Then we have that

$$
\begin{aligned}
\left\|\hat{f}_{1}-\hat{f}_{2}\right\|_{\mathcal{Z}^{n}}^{2} & \leq\left(\left\|f_{1}-f_{2}\right\|_{\mathcal{Z}^{n}}+\left\|f_{1}-\hat{f}_{1}\right\|_{\mathcal{Z}^{n}}+\left\|\hat{f}_{2}-f_{2}\right\|_{\mathcal{Z}^{n}}\right)^{2} \\
& \leq\left(\left\|f_{1}-f_{2}\right\|_{\mathcal{Z}^{n}}+2 \cdot \sqrt{\left|\mathcal{Z}^{n}\right|} \cdot \sqrt{\delta /\left(64 N^{3}\right)}\right)^{2} \\
& \leq \alpha / 50
\end{aligned}
$$

We consider the $Y_{i}$ 's which correspond to $\hat{f}_{1}$ and $\hat{f}_{2}$. Because $\left\|\hat{f}_{1}-\hat{f}_{2}\right\|_{\mathcal{Z}^{n}}^{2} \leq \alpha / 50$, we also have that $\left\|\hat{f}_{1}-\hat{f}_{2}\right\|_{\mathcal{Z}^{n-1}}^{2} \leq \alpha / 50$. From $\mathcal{E}^{n-1}$ we know that $\min \left\{\left\|\hat{f}_{1}-\hat{f}_{2}\right\|_{\widehat{\mathcal{Z}}^{n-1}}^{2}, 4 N W^{2}\right\} \leq 10000 \alpha$. Then from the definition of $Y_{i}$ we have

$$
\left\|\hat{f}_{1}-\hat{f}_{2}\right\|_{\overline{\mathcal{Z}}^n}^{2}=\sum_{i=2}^{n} Y_{i}
$$

Then $\left\|\hat{f}_{1}-\hat{f}_{2}\right\|_{\overline{\mathcal{Z}}^n}^{2}$ can be bounded in the following manner:

$$
\begin{aligned}
\left\|\hat{f}_{1}-\hat{f}_{2}\right\|_{\mathcal{Z}^{n}}^{2} &=\sum_{i=2}^{n} Y_{i} \\
& \leq \sum_{i=2}^{n} \mathbb{E}_{i-1}\left[Y_{i}\right]+\alpha / 100 \\
&=\left\|\hat{f}_{1}-\hat{f}_{2}\right\|_{\mathcal{Z}^{n}}^{2}+\alpha / 100 \\
& \leq 3 \alpha / 100
\end{aligned}
$$

As a result, $\left\|\hat{f}_{1}-\hat{f}_{2}\right\|_{\overline{\mathcal{Z}}^n}^{2}$ can also be bounded:

$$
\begin{aligned}
\left\|f_{1}-f_{2}\right\|_{\overline{\mathcal{Z}}^{n}}^{2} & \leq\left(\left\|\hat{f}_{1}-\hat{f}_{2}\right\|_{\overline{\mathcal{Z}}^{n}}+\left\|f_{1}-\hat{f}_{1}\right\|_{\overline{\mathcal{Z}}^{n}}+\left\|f_{2}-\hat{f}_{2}\right\|_{\overline{\mathcal{Z}}^{n}}\right)^{2} \\
& \leq\left(\left\|\hat{f}_{1}-\hat{f}_{2}\right\|_{\overline{\mathcal{Z}}^{n}}+2 \cdot \sqrt{\left|\overline{\mathcal{Z}}^{n}\right|} \cdot \sqrt{\delta /\left(64 N^{3}\right)}\right)^{2} \\
& \leq \alpha / 25
\end{aligned}
$$

Finally we could bound $\left\|f_{1}-f_{2}\right\|_{\widehat{\mathcal{Z}}^{n}}^{2}$ :

$$
\begin{aligned}
\left\|f_{1}-f_{2}\right\|_{\widehat{\mathcal{Z}}^{n}}^{2} & \leq\left(\left\|f_{1}-f_{2}\right\|_{\overline{\mathcal{Z}}^{n}}+\sqrt{64 N^{3} / \delta} /\left(8 \sqrt{64 N^{3} / \delta}\right)\right)^{2} \\
& \leq \alpha
\end{aligned}
$$

We conclude that for any pair $f_{1}, f_{2} \in \mathcal{F}$ with $\left\|f_{1}-f_{2}\right\|_{\mathcal{Z}^{n}}^{2} \leq \alpha / 100$, it holds that $\left\|f_{1}-f_{2}\right\|_{\widehat{\mathcal{Z}}^{n}}^{2} \leq \alpha$. Thus we must have $\underline{\mathcal{B}}^{n}(\alpha) \subseteq \mathcal{B}^{n}(\alpha)$.

\paragraph{Part 2} $\left(\mathcal{B}^{n}(\alpha) \subseteq \overline{\mathcal{B}}^{n}(\alpha)\right)$ : Consider any pair $f_{1}, f_{2} \in \mathcal{F}$ with $\left\|f_{1}-f_{2}\right\|_{\mathcal{Z}^{n}}^{2}>100 \alpha$. From the definition we know that there exist $\left(\hat{f}_{1}, \hat{f}_{2}\right) \in \mathcal{C}\left(\mathcal{F}, \sqrt{\delta /\left(64 N^{3}\right)}\right) \times \mathcal{C}\left(\mathcal{F}, \sqrt{\delta /\left(64 N^{3}\right)}\right)$ such that $\left\|\hat{f}_{1}-f_{1}\right\|_{\infty},\left\|\hat{f}_{2}-f_{2}\right\|_{\infty} \leq \sqrt{\delta /\left(64 N^{3}\right)}$. Then we have that

$$
\begin{aligned}
\left\|\hat{f}_{1}-\hat{f}_{2}\right\|_{\mathcal{Z}^{n}}^{2} & \geq\left(\left\|f_{1}-f_{2}\right\|_{\mathcal{Z}^{n}}-\left\|f_{1}-\hat{f}_{1}\right\|_{\mathcal{Z}^{n}}-\left\|\hat{f}_{2}-f_{2}\right\|_{\mathcal{Z}^{n}}\right)^{2} \\
& \geq\left(\left\|f_{1}-f_{2}\right\|_{\mathcal{Z}^{n}}-2 \cdot \sqrt{\left|\mathcal{Z}^{n}\right|} \cdot \sqrt{\left.\delta /\left(64 N^{3}\right)\right)}\right)^{2} \\
&>50 \alpha
\end{aligned}
$$

Thus we have $\left\|\hat{f}_{1}-\hat{f}_{2}\right\|_{\mathcal{Z}^{n}}^{2}>50 \alpha$. We consider the $Y_{i}$ 's which correspond to $\hat{f}_{1}$ and $\hat{f}_{2}$. Here we want to prove that $\left\|\hat{f}_{1}-\hat{f}_{2}\right\|_{\widehat{\mathcal{Z}}^{n}}^{2}>40 \alpha$. For the sake of contradicition we assume that $\left\|\hat{f}_{1}-\hat{f}_{2}\right\|_{\widehat{\mathcal{Z}}^{n}}^{2} \leq 40 \alpha$.

\paragraph{Case 1}: $\left\|\hat{f}_{1}-\hat{f}_{2}\right\|_{\mathcal{Z}^{n}}^{2} \leq 2000000 \alpha$. From the definition of $Y_{i}$ we have

$$
\left\|\hat{f}_{1}-\hat{f}_{2}\right\|_{\overline{\mathcal{Z}}^{n}}^2=\sum_{i=2}^{n} Y_{i}
$$

Combined with the former result, we conclude that

$$
\left\|\hat{f}_{1}-\hat{f}_{2}\right\|_{\overline{\mathcal{Z}}^{n}}^{2}=\sum_{i=2}^{n} Y_{i} \geq \sum_{i=2}^{n} \mathbb{E}_{i-1}\left[Y_{i}\right]-\alpha / 100=\left\|\hat{f}_{1}-\hat{f}_{2}\right\|_{\mathcal{Z}^{n}}^{2}-\alpha / 100>50 \alpha-\alpha / 100>49 \alpha
$$

Then we have
$$
\begin{aligned}
& \left\|\hat{f}_{1}-\hat{f}_{2}\right\|_{\widehat{\mathcal{Z}}^{n}}^{2} \geq\left(\left\|\hat{f}_{1}-\hat{f}_{2}\right\|_{\overline{\mathcal{Z}}^{n}}-\sqrt{64 N^{3} / \delta} /\left(8 \sqrt{64 N^{3} / \delta}\right)\right)^{2} \\
& >40 \alpha 
\end{aligned}
$$

which leads to a contradiction.

\paragraph{Case 2}: $\left\|\hat{f}_{1}-\hat{f}_{2}\right\|_{\mathcal{Z}^{n-1}}^{2}>1000000 \alpha$. From $\mathcal{E}^{n-1}$ we deduce that $\left\|\hat{f}_{1}-\hat{f}_{2}\right\|_{\widehat{\mathcal{Z}}^{n-1}}^{2}>100 \alpha$ which directly leads to a contradiction.

\paragraph{Case 3}: $\left\|\hat{f}_{1}-\hat{f}_{2}\right\|_{\mathcal{Z}^{n}}^{2}>2000000 \alpha$ and $\left\|\hat{f}_{1}-\hat{f}_{2}\right\|_{\mathcal{Z}^{n-1}}^{2} \leq 1000000 \alpha$. It is clear that $\left(\hat{f}_{1}\left(z_{n-1}\right)-\hat{f}_{2}\left(z_{n-1}\right)\right)^{2}>$ $1000000 \alpha$. From the definition of sensitivity we know that $z_{n-1}$ will be added into $\overline{\mathcal{Z}}^{n}$ almost surely, which leads to a contradiction.

We conclude that $\left\|\hat{f}_{1}-\hat{f}_{2}\right\|_{\widehat{\mathcal{Z}}^{n}}^{2}>40 \alpha$. Finally we could bound $\left\|f_{1}-f_{2}\right\|_{\widehat{\mathcal{Z}}^{n}}^{2}$ :

$$
\begin{aligned}
\left\|f_{1}-f_{2}\right\|_{\widehat{\mathcal{Z}}^{n}}^{2} & \geq\left(\left\|\hat{f}_{1}-\hat{f}_{2}\right\|_{\widehat{\mathcal{Z}}^{n}}-\left\|f_{1}-\hat{f}_{1}\right\|_{\widehat{\mathcal{Z}}^{n}}-\left\|\hat{f}_{2}-f_{2}\right\|_{\widehat{\mathcal{Z}}^{n}}\right)^{2} \\
& \geq\left(\left\|\hat{f}_{1}-\hat{f}_{2}\right\|_{\widehat{\mathcal{Z}}^{n}}-2 \cdot \sqrt{\left|\widehat{\mathcal{Z}}^{n}\right|} \cdot \sqrt{\delta /\left(64 N^{3}\right)}\right)^{2} \\
&>\alpha
\end{aligned}
$$

We conclude that for any pair $f_{1}, f_{2} \in \mathcal{F}$ with $\left\|f_{1}-f_{2}\right\|_{\mathcal{Z}^{n}}^{2}>10000 \alpha$, it holds that $\left\|f_{1}-f_{2}\right\|_{\widehat{\mathcal{Z}}^{n}}^{2}>$ $100 \alpha$. This implies that $\mathcal{B}^{n}(\alpha) \subseteq \overline{\mathcal{B}}^{n}(\alpha)$.
\end{proof}

Next, we give a bound of the summation of online sensitivity scores.

\begin{lemma} \label{C.5}
(Bound of sensitivity scores). 
We have

$$
\sum_{n=1}^{N-1} \operatorname{sensitivity}_{\mathcal{Z}^{n}, \mathcal{F}}\left(z^{n}\right) \leq C \cdot \operatorname{dim}_{E}(\mathcal{F}, 1 / 4N) \log_2 \left(4 N W^{2}\right) \log N
$$

for some absolute constant $C>0$.
\end{lemma}

\begin{proof}
 Note that $\mathcal{Z}^{n}=\left\{\left(s_{i}, a_{i}\right)\right\}_{i \in[n-1]}$, so $\left|\mathcal{Z}^{n}\right| \leq N$. 

Notice that 
$$
\begin{aligned}
\operatorname{sensitivity}_{\mathcal{Z}^{n}, \mathcal{F}}\left(z^{n}\right)&=\mathop{\text{sup}}\limits_{f_1,f_2\in \mathcal{F}}\frac{{\left(f_1(z)-f_2(z)\right)}^2}{\text{min}\{{||f_1-f_2||}_{\mathcal{Z}^n}^2,4NW^2\}+1}\\ 
&= \mathop{\text{sup}}\limits_{f_1,f_2\in \mathcal{F}}\frac{\left(f_{1}\left(z_{n}\right)-f_{2}\left(z_{n}\right)\right)^{2}}{\left\|f_{1}-f_{2}\right\|_{\mathcal{Z}^{n}}^{2}+1}
\end{aligned}
$$

For each $n \in[N-1]$, let $f_{1}, f_{2} \in \mathcal{F}$ be an arbitrary pair of functions, such that

$$
\frac{\left(f_{1}\left(z_{n}\right)-f_{2}\left(z_{n}\right)\right)^{2}}{\left\|f_{1}-f_{2}\right\|_{\mathcal{Z}^{n}}^{2}+1}
$$

is maximized, and we define $L\left(z_{n}\right)=\left(f_{1}\left(z_{n}\right)-f_{2}\left(z_{n}\right)\right)^{2}$ for such $f_{1}, f_{2}$.

Note that $0 \leq L\left(z_{n}\right) \leq 4 W^{2}$. Let $\mathcal{Z}^{N}=\cup_{\alpha=0}^{\left\lfloor\log _{2}\left(4 W^{2} N\right)\right\rfloor} \mathcal{Z}_{\alpha} \cup \mathcal{Z}_{\infty}$ be a dyadic decomposition with respect to $L(\cdot)$, where for each $0 \leq \alpha \leq\left\lfloor\log _{2}\left(4 W^{2} N\right)\right\rfloor$, we define

$$
\mathcal{Z}_{\alpha}=\left\{z_{n} \in \mathcal{Z}^{N} \mid L\left(z_{n}\right) \in\left(4 W^{2} \cdot 2^{-(\alpha+1)}, 4 W^{2} \cdot 2^{-\alpha}\right]\right\}
$$

and

$$
\mathcal{Z}_{\infty}=\left\{z_{n} \in \mathcal{Z}^{N} \mid L\left(z_{n}\right) \leq 4 W^{2} \cdot 2^{-\left\lfloor\log _{2}\left(4 W^{2} N\right)\right\rfloor-1}\right\}
$$

Therefore, for any $z_{n} \in \mathcal{Z}_{\infty},\ \text{sensitivity}_{\mathcal{Z}^{n}, \mathcal{F}}\left(z_{n}\right) \leq 4 W^{2} \cdot 2^{-\left\lfloor\log _{2}\left(4 W^{2} N\right)\right\rfloor-1}<1 / N$, and thus

$$
\sum_{z_{n} \in \mathcal{Z}_{\infty}} \text{sensitivity}_{\mathcal{Z}^{n}, \mathcal{F}}\left(z_{n}\right) \leq N \cdot \frac{1}{N}=1
$$

For each $\alpha$, let $N_{\alpha}=\left|\mathcal{Z}_{\alpha}\right| / \operatorname{dim}_{E}\left(\mathcal{F}, 4 W^{2} \cdot 2^{-(\alpha+1)}\right)$, and we decompose $\mathcal{Z}_{\alpha}$ into $\left(N_{\alpha}+1\right)$ disjoint subsets, i.e., $\mathcal{Z}_{\alpha}=\cup_{j=1}^{N_{\alpha}+1} \mathcal{Z}_{\alpha}^{j}$, by using the following procedure:

Initialize $Z_{\alpha}^{j}=\{\}$ for all $\mathrm{j}$ and consider each $z_{n} \in \mathcal{Z}_{\alpha}$ sequentially.

For each $z_{n} \in \mathcal{Z}_{\alpha}$, find the smallest $1 \leq j \leq N_{\alpha}$, such that $z_{n}$ is $4 W^{2} \cdot 2^{-(\alpha+1)}$-independent of $\mathcal{Z}_{\alpha}^{j}$ with respect to $\mathcal{F}$.

Set $j=N_{\alpha}+1$ if such $j$ does not exist, use $j\left(z_{n}\right) \in\left[N_{\alpha}+1\right]$ to denote the choice of $j$ for $z_{n}$, and add $z_{n}$ into $Z_{\alpha}^{j}$.

Now, for each $z_{n} \in \mathcal{Z}_{\alpha}, z_{n}$ is $4 W^{2} \cdot 2^{-(\alpha+1)}$-dependent on each of $\mathcal{Z}_{\alpha}^{1}, \mathcal{Z}_{\alpha}^{2}, \cdots, \mathcal{Z}_{\alpha}^{j\left(z_{n}\right)-1}$.

Next, we will show that: For each $z_{n} \in \mathcal{Z}_{\alpha}$,

$$
\text{sensitivity}_{\mathcal{Z}^{n}, \mathcal{F}}\left(z_{n}\right) \leq \frac{4}{j\left(z_{n}\right)}
$$

For any $z_{n} \in \mathcal{Z}_{\alpha}$, let $f_{1}, f_{2} \in \mathcal{F}$ be an arbitrary pair of functions, such that

$$
\frac{\left(f_{1}\left(z_{n}\right)-f_{2}\left(z_{n}\right)\right)^{2}}{\left\|f_{1}-f_{2}\right\|_{\mathcal{Z}^{n}}^{2}+1}
$$

is maximized. Since $z_{n} \in \mathcal{Z}_{\alpha}$, we must have $\left(f_{1}\left(z_{n}\right)-f_{2}\left(z_{n}\right)\right)^{2}>4 W^{2} \cdot 2^{-(\alpha+1)}$, since $z_{n}$ is $4 W^{2} \cdot 2^{-(\alpha+1)}$ dependent on each of $\mathcal{Z}_{\alpha}^{1}, \mathcal{Z}_{\alpha}^{2}, \cdots, \mathcal{Z}_{\alpha}^{j\left(z_{n}\right)-1}$, for each $1 \leq t<j\left(z_{n}\right)$, we have $\left\|f_{1}-f_{2}\right\|_{\mathcal{Z}_{\alpha}^{t}}^{2} \geq 4 W^{2}$. $2^{-(\alpha+1)}$, note that $\mathcal{Z}_{\alpha}^{1}, \mathcal{Z}_{\alpha}^{2}, \cdots, \mathcal{Z}_{\alpha}^{j\left(z_{n}\right)-1} \subset \mathcal{Z}^{n}$ due to the design of the partition procedure. Thus,

$$
\begin{aligned}
& \text{sensitivity}_{\mathcal{Z}^{n}, \mathcal{F}}\left(z_{n}\right) \leq \frac{\left(f_{1}\left(z_{n}\right)-f_{2}\left(z_{n}\right)\right)^{2}}{\left\|f_{1}-f_{2}\right\|_{\mathcal{Z}^{n}}^{2}+1} \leq \frac{4 W^{2} \cdot 2^{-\alpha}}{\left\|f_{1}-f_{2}\right\|_{\mathcal{Z}^{n}}^{2}} \leq \frac{4 W^{2} \cdot 2^{-\alpha}}{\sum_{t=1}^{j(z_{n})-1}\left\|f_{1}-f_{2}\right\|_{\mathcal{Z}_{\alpha}^{t}}^{2}}, \\
& \leq \frac{4 W^{2} \cdot 2^{-\alpha}}{\left(j\left(z_{n}\right)-1\right) \cdot 4 W^{2} \cdot 2^{-(\alpha+1)}} \leq \frac{2}{j\left(z_{n}\right)-1} 
\end{aligned}
$$

Therefore,

$$
\text{sensitivity}_{\mathcal{Z}^{n}, \mathcal{F}}\left(z_{n}\right) \leq \min \left\{\frac{2}{j\left(z_{n}\right)-1}, 1\right\} \leq \frac{4}{j\left(z_{n}\right)}
$$

In addition, by the definition of $4 W^{2} \cdot 2^{-(\alpha+1)}$-independent, we have $\left|\mathcal{Z}_{\alpha}^{j}\right| \leq \operatorname{dim}_{E}\left(\mathcal{F}, 4 W^{2} \cdot 2^{-(\alpha+1)}\right)$ for all $1 \leq j \leq N_{\alpha}$. Therefore,

$$
\begin{aligned}
\sum_{z_{n} \in \mathcal{Z}_{\alpha}} \text { sensitivity }_{\mathcal{Z}^{n}, \mathcal{F}}\left(z_{n}\right) & \leq \sum_{1 \leq j \leq N_{\alpha}}\left|\mathcal{Z}_{\alpha}^{j}\right| \cdot \frac{4}{j}+\sum_{z \in \mathcal{Z}_{\alpha}^{N_{\alpha}+1}} \frac{4}{N_{\alpha}} \\
& \leq 4 \operatorname{dim}_{E}\left(\mathcal{F}, 4 W^{2} \cdot 2^{-(\alpha+1)}\right) \cdot\left(\ln \left(N_{\alpha}\right)+1\right)+\left|\mathcal{Z}_{\alpha}\right| \cdot \frac{4 W^{2} \cdot 2^{-(\alpha+1)}}{\left|\mathcal{Z}_{\alpha}\right|} \\
&=4 \operatorname{dim}_{E}\left(\mathcal{F}, 4 W^{2} \cdot 2^{-(\alpha+1)}\right) \cdot\left(\ln \left(N_{\alpha}\right)+2\right) \\
& \leq 8 \operatorname{dim}_{E}\left(\mathcal{F}, 4 W^{2} \cdot 2^{-(\alpha+1)}\right) \cdot \ln N
\end{aligned}
$$

Now, by the monotonicity of eluder dimension, it follows that:

$$
\begin{aligned}
\sum_{n=1}^{N-1} \text { sensitivity }_{\mathcal{Z}^{n}, \mathcal{F}}\left(z_{n}\right) & \leq \sum_{\alpha=0}^{\left\lfloor\log _{2}\left(4 W^{2} N\right)\right\rfloor} \sum_{z_{n} \in \mathcal{Z}_{\alpha}} \text { sensitivity }_{\mathcal{Z}^{n}, \mathcal{F}}\left(z_{n}\right)+\sum_{z_{n} \in \mathcal{Z}^{\infty}} \text { sensitivity }_{\mathcal{Z}^{n}, \mathcal{F}}\left(z_{n}\right) \\
& \leq 8\left(\left\lfloor\log _{2}\left(4 W^{2} N\right)\right\rfloor+1\right) \operatorname{dim}_{E}(\mathcal{F}, 1 / 4 N) \ln N+1 \\
& \leq 9\left(\left\lfloor\log _{2}\left(4 W^{2} N\right)\right\rfloor+1\right) \operatorname{dim}_{E}(\mathcal{F}, 1 / 4 N) \ln N
\end{aligned}
$$
\end{proof}

The following proposition verifies that $\bigcap_{n=1}^{\infty} \mathcal{E}^{n}$ happens with high probability.
\\

\begin{proposition} \label{C.6}
$$
\mathbb{P}\left(\bigcap_{n=1}^{\infty} \mathcal{E}^{n}\right) \geq 1-\delta / 32
$$
\end{proposition}

\begin{proof}
 For all $n \in[N]$ it holds that

$$
\begin{aligned}
& \mathbb{P}\left(\mathcal{E}^{1} \mathcal{E}^{2} \ldots \mathcal{E}^{n-1}\right)-\mathbb{P}\left(\mathcal{E}^{1} \mathcal{E}^{2} \ldots \mathcal{E}^{n}\right) \\
=& \mathbb{P}\left(\mathcal{E}^{1} \mathcal{E}^{2} \ldots \mathcal{E}^{n-1}\left(\mathcal{E}^{n}\right)^{c}\right) \\
=& \mathbb{P}\left(\mathcal{E}^{1} \mathcal{E}^{2} \ldots \mathcal{E}^{n-1}\left(\bigcap_{j=0}^{\infty} \mathcal{E}^{n}\left(100^{j} \epsilon\right)\right)^{c}\right) \\
=& \mathbb{P}\left(\mathcal{E}^{1} \mathcal{E}^{2} \ldots \mathcal{E}^{n-1} \bigcup_{j=0}^{\infty}\left(\mathcal{E}^{n}\left(100^{j} \epsilon\right)\right)^{c}\right) \\
\leq & \sum_{j=0}^{\infty} \mathbb{P}\left(\mathcal{E}^{1} \mathcal{E}^{2} \ldots \mathcal{E}^{n-1}\left(\mathcal{E}^{n}\left(100^{j} \epsilon\right)\right)^{c}\right) \\
=& \sum_{j \geq 0,100^{j} \epsilon \leq 4 N W^{2}}\mathbb{P}\left(\mathcal{E}^{1} \mathcal{E}^{2} \ldots \mathcal{E}^{n-1}\left(\mathcal{E}^{n}\left(100^{j} \epsilon\right)\right)^{c}\right)
\end{aligned}
$$

where the last equality holds because $\mathbb{P}\left(\mathcal{E}^{n}(\alpha)\right)=1$ while $\alpha>4 N W^{2}$. Combining this with Lemma \ref{C.4} yields

$$
\mathbb{P}\left(\mathcal{E}^{1} \mathcal{E}^{2} \ldots \mathcal{E}^{n-1}\right)-\mathbb{P}\left(\mathcal{E}^{1} \mathcal{E}^{2} \ldots \mathcal{E}^{n}\right) \leq \delta /\left(32 N^{2}\right) \cdot\left(\log \left(4 N W^{2} / \epsilon\right)+2\right) \leq \delta /(32 N)
$$

thus

$$
\begin{aligned}
& \mathbb{P}\left(\bigcap_{n=1}^{N} \mathcal{E}^{n}\right) \\
=& 1-\sum_{n=1}^{N}\left(\mathbb{P}\left(\mathcal{E}^{1} \mathcal{E}^{2} \ldots \mathcal{E}^{n-1}\right)-\mathbb{P}\left(\mathcal{E}^{1} \mathcal{E}^{2} \ldots \mathcal{E}^{n}\right)\right) \\
\geq & 1-N \cdot(\delta / 32 N) \\
=& 1-\delta / 32
\end{aligned}
$$
\end{proof}

With Lemma \ref{C.5}, we are now ready to prove:
\newline

\begin{proposition} \label{C.7}
With probability at least $1-\delta / 8$, the following statements hold:

(i) The subsampled dataset ${\widehat{\mathcal{Z}}}^n$ changes for at most

$$
S_{\max }=C \cdot \log \left(N \mathcal{N}\left(\mathcal{F}, \sqrt{\delta / 64 N^{3}}\right) / \delta\right) \cdot \operatorname{dim}_{E}(\mathcal{F}, 1 / N) \cdot \log ^{2} N
$$

where $C>0$ is some absolute constant.

(ii) For any $n \in[N],\left|\widehat{\mathcal{Z}}^{n}\right| \leq 64 N^{3} / \delta$.
\end{proposition}
\begin{proof}
Conditioning on $\mathcal{E}^{n}$, we have

$$
\mathbb{I}\left\{\mathcal{E}^{n}\right\} \cdot \text{sensitivity}_{\widehat{\mathcal{Z}}^{n}, \mathcal{F}}\left(z_{n}\right)\leq C\cdot\text{sensitivity}_{\mathcal{Z}^{n}, \mathcal{F}}\left(z_{n}\right)
$$

for some constant $C>0$ according to Lemma \ref{C.3}. By definition of $p_{z}$ we have

$$
p_{z} \lesssim \operatorname{sensitivity}_{\widehat{\mathcal{Z}}, \mathcal{F}}(z) \cdot \log \left(N \mathcal{N}\left(\mathcal{F}, \sqrt{\delta / 64 N^{3}}\right) / \delta\right)
$$

thus by Lemma \ref{C.5} we have

$$
\begin{aligned}
\sum_{n=1}^{N-1} \mathbb{I}\left\{\mathcal{E}^{n}\right\} \cdot p_{z_{n}} & \lesssim \sum_{n=1}^{N-1} \mathbb{I}\left\{\mathcal{E}^{n}\right\} \cdot \text{sensitivity}_{\widehat{\mathcal{Z}}^{n}, \mathcal{F}}\left(z_{n}\right) \cdot \log \left(N \mathcal{N}\left(\mathcal{F}, \sqrt{\delta / 64 N^{3}}\right) / \delta\right) \\
& \lesssim \sum_{n=1}^{N-1} \operatorname{sensitivity}_{\mathcal{Z}^{n}, \mathcal{F}}\left(z_{n}\right) \cdot \log \left(N \mathcal{N}\left(\mathcal{F}, \sqrt{\delta / 64 N^{3}}\right) / \delta\right) \\
& \lesssim \log \left(N \mathcal{N}\left(\mathcal{F}, \sqrt{\delta / 64 N^{3}}\right) / \delta\right) \operatorname{dim}_{E}(\mathcal{F}, 1 / N) \log ^{2} N
\end{aligned}
$$

and by choosing $C$ in the proposition appropriately, we may assume that

$$
\sum_{n=1}^{N-1} \mathbb{I}\left\{\mathcal{E}^{n}\right\} \cdot p_{z_{n}} \leq S_{\max } / 3
$$

For $2 \leq n \leq N$, define random variables $\left\{X_{n}\right\}$ as

$$
X_{n}= \begin{cases}\mathbb{I}\left\{\mathcal{E}^{n-1}\right\} & \hat{z}_{n-1} \text { is added into } \widehat{\mathcal{Z}}^{n} \\ 0 & \text { otherwise }\end{cases}
$$

Then $X_{n}$ is adapted to the filtration $\mathcal{F}_{n}$. We have that $\mathbb{E}_{n-1}\left[X_{n}\right]=p_{z_{n-1}} \cdot \mathbb{I}\left\{\mathcal{E}^{n-1}\right\}$ and $\mathbb{E}_{n-1}\left[\left(X_{n}-\mathbb{E}_{n-1}\left[X_{n}\right]\right)^{2}\right]=\mathbb{I}\left\{\mathcal{E}^{n-1}\right\} \cdot p_{z_{n-1}}\left(1-p_{z_{n-1}}\right)$. Note that $X_{n}-\mathbb{E}_{n-1}\left[X_{n}\right]$ is a martingale difference sequence with respect to $\mathcal{F}_{n}$ and

$$
\begin{aligned}
\sum_{n=2}^{N} \mathbb{E}_{n-1}\left[\left(X_{n}-\mathbb{E}_{n-1}\left[X_{n}\right]\right)^{2}\right]&=\sum_{n=2}^{N} \mathbb{I}\left\{\mathcal{E}^{n}\right\} p_{z_{n-1}}\left(1-p_{z_{n-1}}\right) \leq \sum_{n=2}^{N} \mathbb{I}\left\{\mathcal{E}_{n-1}\right\} \cdot p_{z_{n-1}} \leq S_{\max } / 3 \\
\sum_{n=2}^{N} \mathbb{E}_{n-1}\left[X_{n}\right]&=\sum_{n=2}^{n} p_{z_{n-1}} \mathbb{I}\left\{\mathcal{E}_{n-1}\right\} \leq S_{\max } / 3
\end{aligned}
$$

thus by applying Freedman's inequality (Lemma \ref{C.1}), we deduce that

$$
\begin{aligned}
& \mathbb{P}\left\{\sum_{n=2}^{N} X_{n} \geq S_{\max }\right\} \\
\leq & \mathbb{P}\left\{\left|\sum_{n=2}^{N}\left(X_{n}-\mathbb{E}_{n-1}\left[X_{n}\right]\right)\right| \geq 2 S_{\max } / 3\right\} \\
\leq & 2 \exp \left\{-\frac{\left(2 S_{\max } / 3\right)^{2} / 2}{S_{\max } / 3+2 S_{\max } / 9}\right\} \\
\leq & \delta /(32 N)
\end{aligned}
$$

With a union bound we know that with probability at least $1-\delta / 32$,

$$
\sum_{n=2}^{N} X_{n}<S_{\max }
$$

We condition on the event above and $\bigcap_{n=1}^{N} \mathcal{E}^{n}$. In this case, we add elements into $\widehat{\mathcal{Z}}^{n}$ for at most $S_{\text {max }}$ times. Combining the result above with Lemma \ref{C.2} completes the proof.

\end{proof}

\section{Concentration of Bonuses} \label{App D}

Before bounding the bonuses, we need the following concentration inequality proved in \citep{beygelzimer2011contextual}.

\begin{lemma} \label{D.3} 
(Bernstein for Martingales).

Consider a sequence of random variables $X_1,X_2,\cdots, X_T$. Assume that for all $t$, $X_t\leq R$, and $\mathbb{E}_t[X_t]\overset{def}{=}\mathbb{E}[X_t|X_1,\cdots,X_{t-1}]=0$. Then for any $\delta>0$, there exists constant $c_1,c_2$, such that
$$
\mathbf{P}\left(\sum_{t=1}^{T} X_{t} \leq c_1 \times \sqrt{\sum_{t=1}^{T} \mathbb{E}_t[X_t^2]\ln \frac{1}{\delta}}+c_2 \times \ln \frac{1}{\delta}\right) \geq 1-\delta
$$
\end{lemma}

\begin{lemma} \label{D.1}
(Bound of Indicators).
For any episode $n$ during the execution of the algorithm, with probability $1-\delta/2$,
\begin{equation}
	\begin{aligned}
	\sum_{n=1}^N \mathbb{E}_{(s,a)\sim d^{n}}b_1^n(s,a)\leq \widetilde{O}\left(\frac{\sqrt{N{d}^2\epsilon}}{(1-\gamma)\beta}\right)
	\end{aligned}
\end{equation}
where $d =\text{dim}_E(\mathcal{F},1/N)$.
\end{lemma}

\begin{proof}
\begin{equation}
	\begin{aligned}
	\sum_{n=1}^N\mathbb{E}_{(s,a)\sim d^{n}}b_1^n(s,a)&\leq\frac{3}{1-\gamma}\sum_{n=1}^N\mathbb{E}_{(s,a)\sim d^{n}}\textbf{1}\{\omega(\widehat{\mathcal{F}}^n,s,a)\geq\beta\}\\ &=\frac{3}{1-\gamma}\sum_{n=1}^N\mathbb{E}_{(s,a)\sim d^{n}}\textbf{1}\{\frac{1}{\beta}\omega(\widehat{\mathcal{F}}^n,s,a)\geq1\}\\& \leq\frac{3}{1-\gamma}\cdot\frac{1}{\beta}\sum_{n=1}^N\mathbb{E}_{(s,a)\sim d^{n}}\omega(\widehat{\mathcal{F}}^n,s,a)\\ &\leq \widetilde{O}\left(\frac{\sqrt{N{d}^2\epsilon}}{(1-\gamma)\beta}\right)\ (\text{by Lemma \ref{D.2}})
	\end{aligned}
\end{equation}
\end{proof}

\begin{lemma} \label{D.2}
(Bound of Bonuses).
For any episode $n$ during the execution of the algorithm, with probability $1-\delta/2$

\begin{equation}
	\begin{aligned}
	\sum_{n=1}^N\mathbb{E}_{(s,a)\sim d^{n}}\omega(\widehat{\mathcal{F}}^n,s,a)\leq O\left(\sqrt{N{d}^2\epsilon}+\ln(\frac{2}{\delta})\right)=\widetilde{O}\left(\sqrt{N{d}^2\epsilon}\right)
	\end{aligned}
\end{equation}
where $d =\text{dim}_E(\mathcal{F},1/N)$.
\end{lemma}

\begin{proof}
We define the random dataset $\mathcal{D}_{1:n}$ to represent all the information at the beginning of iteration $n$ of the algorithm. Then we define
$$
\xi_n=\mathbb{E}_{(s,a)\sim d^{n}}[\omega(\widehat{\mathcal{F}}^n,s,a)|\mathcal{D}_{1:n}]-\omega(\widehat{\mathcal{F}}^n,s_n,a_n)
$$
and let
$$
A=\sum_{n=1}^N\mathbb{E}_{(s,a)\sim d^{n}}[\omega(\widehat{\mathcal{F}}^n,s,a)|\mathcal{D}_{1:n}]=\sum_{n=1}^N\omega(\widehat{\mathcal{F}}^n,s_n,a_n)+\sum_{n=1}^N\xi_n
$$
Now we bound $\sum_{n=1}^N\omega(\widehat{\mathcal{F}}^n,s_n,a_n)$:\newline
We condition on the event in the Lemma \ref{C.6}, we have 
$$
\omega(\widehat{\mathcal{F}}^n,s,a)\leq \mathop{\text{sup}}\limits_{f_1,f_2\in\mathcal{F},||f_1-f_2||_{\mathcal{Z}^n}^2\leq100\epsilon}|f_1(s,a)-f_2(s,a)|\overset{def}{=}\bar{b}^n(s,a)
$$
For any given $\alpha>0$, let $\mathcal{L}=\{(s_n,a_n)|n\in[N],\bar{b}^n(s_n,a_n)>\alpha\}$, let $|\mathcal{L}|=L$.\newline Next we show that there exists $z_k:=(s_k,a_k)\in\mathcal{L}$, such that $(s_k,a_k)$ is $\alpha$-dependent on at least $N=L/\text{dim}_E(\mathcal{F},\alpha)-1$ disjoint subsequences in $\mathcal{Z}^k\cap\mathcal{L}$. We decompose the $\mathcal{L}=\cup_{j=1}^{N+1}\mathcal{L}^j$ by the following procedure. We initialize $\mathcal{L}^j=\{\}$ for all $j$ and consider $z_k\in\mathcal{L}$ sequentially. For each $z_k\in\mathcal{L}$, we find the smallest $j\ (1\leq j\leq N)$, such that $z_k$ is $\alpha$-independent on $\mathcal{L}^j$ with respect to $\mathcal{F}$. We set $j=N+1$ if such $j$ does not exist. We add $z_k$ into $\mathcal{L}^j$ afterwards. When the decomposition of $\mathcal{L}$ is finished, $\mathcal{L}^{N+1}\neq\emptyset$, as $\mathcal{L}^j$ contains at most $\text{dim}_E(\mathcal{F},\alpha)$ elements for $j\in [N]$. For any $z_k\in\mathcal{L}^{N+1}$, $z_k$ is $\alpha$-dependent on at least $N=L/\text{dim}_E(\mathcal{F},\alpha)-1$ disjoint subsequences in $\mathcal{Z}^k\cap\mathcal{L}$. \newline On the other hand, there exists $f_1,f_2\in\mathcal{F}$ with $||f_1-f_2||_{\mathcal{Z}^k}^2\leq100\epsilon$, such that $|f_1(s,a)-f_2(s,a)|>\alpha$. So we have:
$$
(L/\text{dim}_E(\mathcal{F},\alpha)-1)\cdot\alpha^2\leq||f_1-f_2||_{\mathcal{Z}^k}^2\leq100\epsilon
$$
which implies
$$
L\leq(\frac{100\epsilon}{\alpha^2}+1)\text{dim}_E(\mathcal{F},\alpha)
$$
Now let $b_1\geq b_2\geq\cdots b_N$ to be a permulation of $\{\bar{b}^n(s_n,a_n)\}_{n=1}^N$. For any $b_n\geq\frac{1}{N}$, we have 
$$
n\leq(\frac{100\epsilon}{b_n^2}+1)\text{dim}_E(\mathcal{F},b_n)\leq(\frac{100\epsilon}{b_n^2}+1)\text{dim}_E(\mathcal{F},\frac{1}{N})
$$
which implies that
$$
b_n\leq\left(\frac{n}{\text{dim}_E(\mathcal{F},\frac{1}{N})}-1\right)^{-\frac{1}{2}}\sqrt{100\epsilon},\  \text{when}\  b_n\geq 1/N
$$
Moreover, we have $b_n\leq 2W$, so 
\begin{equation}
	\begin{aligned}
\sum_{n=1}^N b_n&\leq N\cdot\frac{1}{N}+2W\cdot\text{dim}_E(\mathcal{F},1/N)+\sum_{\text{dim}_E(\mathcal{F},1/N)<n\leq N}\left(\frac{n}{\text{dim}_E(\mathcal{F},\frac{1}{N})}-1\right)^{-\frac{1}{2}}\sqrt{100\epsilon}\\ &\leq1+2W\cdot\text{dim}_E(\mathcal{F},1/N)+C\cdot\sqrt{\text{dim}_E(\mathcal{F},1/N)\cdot N\cdot\epsilon}
	\end{aligned}
\end{equation}
For simplicity, we denote $d:=\text{dim}_E(\mathcal{F},1/N)$, then $\sum_{n=1}^N\omega(\widehat{\mathcal{F}}^n,s_n,a_n)\leq O(\sqrt{N{d}^2\epsilon})$.

Then we will bound the sum of noise terms:
\begin{equation}
	\begin{aligned}
	\sum_{n=1}^N\mathbb{E}_{(s,a)\sim d^n}[\xi_n^2|\mathcal{D}_{1:n}]&=\sum_{n=1}^N\mathbb{E}_{(s,a)\sim d^n}[\omega^2(\widehat{\mathcal{F}}^n,s,a)|\mathcal{D}_{1:n}]\\ &\leq 2W\cdot \sum_{n=1}^N\mathbb{E}_{(s,a)\sim d^n}[\omega(\widehat{\mathcal{F}}^n,s,a)|\mathcal{D}_{1:n}]
	\end{aligned}
\end{equation}

Now using the Lemma \ref{D.3} (Bernstein for Martingales) gives with probability at least $1-\frac{\delta}{2}$ for some constant $c$

\begin{equation}
	\begin{aligned}
	\sum_{n=1}^N\xi_n &\leq c\times\left(\sqrt{2\sum_{n=1}^N\mathbb{E}_{(s,a)\sim d^n}[\xi_n^2|\mathcal{D}_{1:n}]\ln(2/\delta)}+\frac{\ln(2/\delta)}{3}\right)\\&=c\times\left(\sqrt{4WA\ln(2/\delta)}+\frac{\ln(2/\delta)}{3}\right)
	\end{aligned}
\end{equation}
Solving for A finally gives with high probability
\begin{equation}
	\begin{aligned}
	A=O\left(\sqrt{N{d}^2\epsilon}+\ln(\frac{2}{\delta})\right)
	\end{aligned}
\end{equation}
\end{proof}

\section{Analysis of Policy Evaluation Oracle} \label{App E}
In this section, we provide the theoretical guarantee of our policy evaluation oracle using importance sampling technique.
\begin{definition} \label{E.1}
(Importance Sampling Estimator).
Let $t$ be a positive discrete random variable with probability mass function $\mathbf{P}(t=\tau)=\gamma^{\tau-1}(1-\gamma)$, and let $\left\{\left(s_{\tau}, a_{\tau}, r_{\tau}\right)\right\}_{\tau=1, \ldots, t}$ be a random trajectory of length $t$ obtained by following a fixed "behavioral" policy $\underline{\pi}$ from $(s, a)$. The importance sampling estimator of the target policy $\pi$ is:
$$
\left(\Pi_{\tau=2}^{t} \frac{\pi\left(s_{\tau}, a_{\tau}\right)}{\underline{\pi}\left(s_{\tau}, a_{\tau}\right)}\right) \frac{r_{t}}{1-\gamma} .
$$
\end{definition}

Notice that our inner loop solves the bonus-added MDP problem, so $r_t$ is replaced by $G_t$ in the following formula.

$$ G_t=\left\{
\begin{aligned}
&\frac{1}{1-\gamma}[r_t+b(s_t,a_t)], \ \text{if} \ t
\geq 2 \\
&\frac{1}{1-\gamma}[r_t], \ \text{if}\  t=1 
\end{aligned}
\right.
$$

\begin{definition} \label{E.2}
 We define $B=\frac{3}{1-\gamma}$, $G_{\text{max}}=\frac{2+B}{(1-\gamma)}$,
$\delta_1=\gamma^{\alpha}$
\end{definition}

\begin{remark}

From the definition of bonus function, we know that $0\leq b(\cdot,\cdot)\leq B$. In addition, the random return from a single Monte Carlo trajectory $\frac{G_t}{1-\gamma}$ has a deterministic upper bound $G_{\text{max}}$. For a concise bound, we can assume $2G_{\text{max}}\leq W$ in the following proof.

\end{remark}

\begin{lemma} (Stability of Importance Sampling Estimator) \label{E.3}
When 
$$
k-\underline{k} \leq \kappa \stackrel{\text { def }}{=} \frac{(1-\gamma) \ln 2}{2 \ln \left(1 / \delta_1\right) \eta(B+W)},
$$
then with probability $1-\delta_1$, 
$$
\left(\Pi_{\tau=2}^{t} \frac{\pi\left(s_{\tau}, a_{\tau}\right)}{\underline{\pi}\left(s_{\tau}, a_{\tau}\right)}\right) \frac{G_{t}}{1-\gamma}\leq 2 G_{\text{max}}
$$
\end{lemma}
\begin{remark} \label{kappa}
This lemma indicates that if we want to get a stable importance sampling estimator during policy evaluation process, $\kappa$ should be $O(\sqrt{K})$ \ (since $\eta$ has an order of $O(\frac{1}{\sqrt{K}})$ by Lemma \ref{B.7}).
\end{remark}

\begin{proof}
This lemma combines the results of Appendix G in \citep{zanette2021cautiously}.
First of all, we need to figure out the policy form on the known set. In fact, we have the following conclusion.
$$
\forall(s, a), \quad \pi_k(a \mid s)=\pi_{\underline{k}}(a \mid s) \times \frac{e^{c(s, a)}}{\sum_{a^{\prime}} \pi_{\underline{k}}\left(a^{\prime} \mid s\right) e^{c\left(s, a^{\prime}\right)}}
$$
where
$$
c(s,a)=\eta\cdot\sum_{i=\underline{k}}^{k-1}[b(s,a)+f_i(s,a)]
$$
We assume $k>\underline{k}$, then according to the algorithm,
$$
\begin{aligned}
\pi_k(\cdot|s)&\propto\pi_{k-1}(\cdot|s)e^{\eta[f_{k-1}(s,\cdot)+b(s,\cdot)]}\\ &\propto \pi_{\underline{k}}(\cdot|s)e^{\eta\sum_{i=\underline{k}}^{k-1}[f_i(s,\cdot)+b(s,\cdot)]}
\end{aligned}
$$
We define 
$$
c(s,a)=\eta\cdot\sum_{i=\underline{k}}^{k-1}[b(s,a)+f_i(s,a)]
$$
So the desired result is obtained.\\
To simplify the notation, we use $c$ to denote $\sup_{(s,a)}|c(s,a)|$.
Then the following chain of inequalities is true.
$$
e^{-2 c} \leq \frac{e^{-c}}{\sum_{a^{\prime}} \underline{\pi}\left(a^{\prime} \mid s\right) e^{c}} \leq \frac{\pi(a \mid s)}{\underline{\pi}(a \mid s)}\leq \frac{e^{c(s, a)}}{\sum_{a^{\prime}} \pi\left(a^{\prime} \mid s\right) e^{c\left(s, a^{\prime}\right)}}\leq \frac{e^{c}}{\sum_{a^{\prime}} \underline{\pi}\left(a^{\prime} \mid s\right) e^{-c}} =e^{2c}
$$
So we can bound the policy ratio.
$$
e^{-2 c} \leq \sup _{(s, a)} \frac{\pi(a \mid s)}{\pi(a \mid s)} \leq e^{2 c}
$$

Notice that
$$
\eta\cdot\kappa\cdot(B+W)\geq \sup_{(s,a)}|c(s,a)|
$$
Then we have
$$
c=\sup_{(s,a)}|c(s,a)|\leq \frac{(1-\gamma)\ln2}{2\ln(1/\delta_1)}
$$
Remember that $t$ is small with high probability:

$$
\begin{aligned}
\mathbf{P}(t>\alpha) &=\sum_{t=\alpha+1}^{\infty} \gamma^{\alpha-1}(1-\gamma) \\
&=\gamma^{\alpha} \sum_{t=0}^{\infty} \gamma^{\alpha}(1-\gamma) \\
&=\gamma^{\alpha} \stackrel{\text { def }}{=} \delta_1 .
\end{aligned}
$$

This implies

$$
\alpha=\frac{\ln \delta_1 }{\ln \gamma}=\frac{\ln 1 / \delta_1 }{\ln 1 / \gamma} \leq \frac{\ln 1 / \delta_1 }{1-\gamma}
$$

In the complement of the above event:

$$
\left(\sup _{(s, a)} \frac{\pi(a \mid s)}{\pi(a \mid s)}\right)^{t-1} \leq e^{2(\alpha-1) c}\leq e^{\frac{(\alpha-1)(1-\gamma)\ln2}{\ln(1/\delta_1)}}\leq 2 .
$$

Then with probability at least $1-\delta_1$ if the importance sampling ratio is upper bounded

$$
\Pi_{\tau=2}^{t} \frac{\pi\left(s_{\tau}, a_{\tau}\right)}{\underline{\pi}\left(s_{\tau}, a_{\tau}\right)} \leq\left(\sup _{(s, a)} \frac{\pi(a \mid s)}{\pi(a \mid s)}\right)^{t-1} \leq 2
$$
And $\frac{G_{t}}{1-\gamma}$ is bounded by $G_{\text{max}}$ in absolute value, which guarantees our result.
\end{proof}

\begin{lemma} \label{E.4}
(Concentration of Width Function).
If we set 
\begin{equation} \label{eq 34}
	\begin{aligned}
	\frac{1}{100}\epsilon=\frac{3}{2}C_1N\cdot\epsilon_{\text{stat}}+20NW\epsilon_1+\frac{1}{2}C_2\cdot\ln\left(\frac{N\mathcal{N}(\Delta\mathcal{F},2\epsilon_1)}{\delta}\right)
	\end{aligned}
\end{equation}
where $\epsilon_1$ denotes the function cover radius, $C_1$, $C_2$ is some constant defined in the following proof, $\epsilon_{\text{stat}}$ will be determined in Lemma \ref{E.5}.

Then with probalility at least $1-\frac{1}{2}\delta$, for all $n\in [N]$
\begin{equation}
	\begin{aligned}
	||\Delta f_k||_{\widehat{\mathcal{Z}}^n}^2\leq \epsilon
	\end{aligned}
\end{equation}
\end{lemma}

\begin{proof} 
Conditioned on the Proposition \ref{C.6}, we only need to prove
\begin{equation}
	\begin{aligned}
	||\Delta f_k||_{\mathcal{Z}^n}^2\leq 100\epsilon
	\end{aligned}
\end{equation}

Let $\mathcal{C}\left(\Delta \mathcal{F}, 2 \epsilon_{1}\right)$ be a cover set of $\Delta \mathcal{F}$. Then for every $\Delta f \in \Delta \mathcal{F}$, there exists a $\Delta g \in \mathcal{C}\left(\Delta \mathcal{F}, 2 \epsilon_{1}\right)$ such that $\|\Delta f-\Delta g\|_{\infty} \leq 2 \epsilon_{1}$. 
Consider a fixed $\Delta g \in \mathcal{C}\left(\Delta \mathcal{F}, 2 \epsilon_{1}\right)$, we define n random variables:
$$
X_i=\frac{1}{8W^2}\left(\left(\Delta g\left(s_i,a_i\right)\right)^{2}-\mathbb{E}_{(s, a) \sim d^{\pi^i}}\left[(\Delta g(s, a))^{2}\right]\right), i \in[n]
$$

Notice that for all $i\in [n]$, $X_i\leq 1$, $\mathbb{E}_i[X_i]=0$, and 
$$
\mathbb{E}_i[X_i^2]\leq \mathbb{E}_i[|X_i|]\leq \mathbb{E}_i\left[\frac{\left(\Delta g(s_i,a_i)\right)^2}{4W^2}\right]=\frac{1}{4W^2}\mathbb{E}_{(s,a)\sim d^{\pi^i}}\left(\Delta g(s,a)\right)^2
$$

Then by using Lemma \ref{D.3} (Bernstein for Martingales), we have the following inequality: With probability at least $1-\delta_2$,
$$
\sum_{i=1}^n X_i\leq c_1\times \sqrt{\frac{\ln \frac{1}{\delta_2}}{4W^2}\sum_{i=1}^n\mathbb{E}_{(s,a)\sim d^{\pi^i}}\left(\Delta g(s,a)\right)^2}+c_2\times \ln\frac{1}{\delta_2}
$$
which means
$$
\frac{1}{n}\sum_{i=1}^n\left[\left(\Delta g(s_i,a_i)\right)^2-\mathbb{E}_{(s, a) \sim d^{\pi^i}}(\Delta g(s, a))^{2}\right]\leq c\times\left(\sqrt{\frac{\ln \frac{1}{\delta_2}}{n^2}\sum_{i=1}^n\mathbb{E}_{(s,a)\sim d^{\pi^i}}\left(\Delta g(s,a)\right)^2}+\frac{1}{n}\ln\frac{1}{\delta_2}\right)
$$

We now proof that if $\lambda=C\cdot\ln\left(\frac{1}{\delta_2}\right)$, which $C$ is a constant appropriate large, then 
$$
c\times\left(\sqrt{\frac{\ln \frac{1}{\delta_2}}{n^2}\sum_{i=1}^n\mathbb{E}_{(s,a)\sim d^{\pi^i}}\left(\Delta g(s,a)\right)^2}+\frac{1}{n}\ln\frac{1}{\delta_2}\right)\leq \frac{1}{2}\left(\frac{1}{n}\left(\sum_{i=1}^n\mathbb{E}_{(s,a)\sim d^{\pi^i}}\left(\Delta g(s,a)\right)^2\right)+\frac{\lambda}{n}\right)
$$
To simplify the notation, we define $A=\sum_{i=1}^n\mathbb{E}_{(s,a)\sim d^{\pi^i}}\left(\Delta g(s,a)\right)^2$.

\paragraph{Case 1: $A\leq \lambda$}
According to the selection of $\lambda$, there exists constant $c',c''$ appropriate small, such that
$$
\frac{1}{n}\ln\left(\frac{1}{\delta_2}\right)\leq c'\cdot\left(\frac{\lambda}{n}\right)
$$
$$
\sqrt{\frac{\lambda}{n^2}\ln\left(\frac{1}{\delta_2}\right)}\leq c''\cdot\left(\frac{\lambda}{n}\right)
$$
Then
$$
\textbf{LHS}\leq c\times\left(\sqrt{\frac{\lambda}{n^2}\ln\left(\frac{1}{\delta_2}\right)}+\frac{1}{n}\ln\left(\frac{1}{\delta_2}\right)\right)\leq c\times (c'+c'')\left(\frac{\lambda}{n}\right)\leq \frac{1}{2}\left(\frac{\lambda}{n}+\frac{A}{n}\right)=\textbf{RHS}
$$

\paragraph{Case 2: $A\geq \lambda$} there also exists constant $c',c''$ appropriate small, such that
$$
\frac{1}{n}\ln\left(\frac{1}{\delta_2}\right)\leq c'\left(\frac{A}{n}\right)
$$
$$
\sqrt{\frac{A}{n^2}\ln\left(\frac{1}{\delta_2}\right)}\leq c''\cdot\left(\frac{A}{n}\right)
$$
Then
$$
\textbf{LHS}\leq c\times\left(\sqrt{\frac{A}{n^2}\ln\left(\frac{1}{\delta_2}\right)}+\frac{1}{n}\ln\left(\frac{1}{\delta_2}\right)\right)\leq c\times (c'+c'')\left(\frac{A}{n}\right)\leq \frac{1}{2}\left(\frac{\lambda}{n}+\frac{A}{n}\right)=\textbf{RHS}
$$
After taking the union bound on the function cover $\mathcal{C}(\Delta\mathcal{F},2\epsilon_1)$, we have the following result:
With probalility at least $1-N\mathcal{N}(\Delta\mathcal{F},2\epsilon_1)\delta_2\overset{def}{=}1-\frac{1}{8}\delta$, by setting $\lambda=C\cdot\ln\left(\frac{N\mathcal{N}(\Delta\mathcal{F},2\epsilon_1)}{\delta}\right)$, we have
$$
\forall n, \forall \Delta g \in \mathcal{C}(\Delta\mathcal{F},2\epsilon_1),\sum_{i=1}^n\left[\left(\Delta g(s_i,a_i)\right)^2-\mathbb{E}_{(s, a) \sim d^{\pi^i}}(\Delta g(s, a))^{2}\right]\leq \frac{1}{2}\left(\sum_{i=1}^n\mathbb{E}_{(s, a) \sim d^{\pi^i}}(\Delta g(s, a))^{2}+\lambda \right)
$$

Next, we transform to an arbitrary function $\Delta f \in \Delta \mathcal{F}$. Since there exists a $\Delta g \in \mathcal{C}\left(\Delta \mathcal{F}, 2 \epsilon_{1}\right)$ such that $\left\|\Delta f-\Delta g\right\|_{\infty} \leq 2 \epsilon_{1}$, we have that for all $i \in[n]$,

$$
\begin{aligned}
&\left|\left(\Delta f\left(s_i, a_i\right)\right)^{2}-\left(\Delta g\left(s_i, a_i\right)\right)^{2}\right| \\
&\left.=\left|\Delta f\left(s_i, a_i\right)-\Delta g\left(s_i, a_i\right)\right| \cdot \mid \Delta f\left(s_i, a_i\right)+\Delta g\left(s_i, a_i\right)\right) \mid \leq 8 W \epsilon_{1}
\end{aligned}
$$

and

$$
\begin{aligned}
&\left|\mathbb{E}_{(s, a) \sim d^{{\pi}^i}}\left[\left(\Delta f(s, a)\right)^{2}\right]-\mathbb{E}_{(s, a) \sim d^{{\pi}^i}}\left[(\Delta g(s, a))^{2}\right]\right| \\
&\leq \mathbb{E}_{(s, a) \sim d^{{\pi}^i}}\left|\Delta f(s, a)-\Delta g(s, a)\right| \cdot\left|\Delta f(s, a)+\Delta g(s, a)\right| \leq 8 W \epsilon_{1}
\end{aligned}
$$
Therefore,
$$
\begin{aligned}
&\sum_{i=1}^n\left[\left(\Delta f(s_i,a_i)\right)^2-\mathbb{E}_{(s, a) \sim d^{\pi^i}}(\Delta f(s, a))^{2}\right]\\ \leq & \left|\sum_{i=1}^n\left[\left(\Delta f(s_i, a_i)\right)^{2}-\left(\Delta g(s_i, a_i)\right)^{2}\right]\right|+\left|\sum_{i=1}^n\left[\left(\Delta g(s_i, a_i)\right)^{2}-\mathbb{E}_{(s, a) \sim d^{\pi^i}}(\Delta g(s, a))^{2}\right]\right| \\ + & \left|\sum_{i=1}^n\left[\mathbb{E}_{(s, a) \sim d^{\pi^i}}(\Delta f(s, a))^{2}-\mathbb{E}_{(s, a) \sim d^{\pi^i}}(\Delta g(s, a))^{2}\right]\right| \\ \leq &  \frac{1}{2}\left(\sum_{i=1}^n\mathbb{E}_{(s, a) \sim d^{\pi^i}}(\Delta g(s, a))^{2}+\lambda\right)+16nW\epsilon_1 \\ \leq & \frac{1}{2}\left(\sum_{i=1}^n\mathbb{E}_{(s, a) \sim d^{\pi^i}}(\Delta f(s, a))^{2}+8nW\epsilon_1+\lambda\right)+16nW\epsilon_1
\end{aligned}
$$

Then,
$$
\forall n \in [N], \forall \Delta f \in \Delta \mathcal{F}, ||\Delta f||_{\mathcal{Z}^n}^2\leq \frac{3}{2}\sum_{i=1}^n\mathbb{E}_{(s, a) \sim d^{\pi^i}}(\Delta f(s, a))^{2}+\frac{1}{2}\lambda+20nW\epsilon_1
$$

Then we have with probability at least $1-\frac{1}{8}\delta$, 
$$
\left\|\Delta f_{k}\right\|_{\mathcal{Z}^{n}}^{2} \leq \frac{3}{2}n \cdot \mathbb{E}_{\rho_{\mathrm{cov}}^{n}}\left[\left(\Delta f_{k}\right)^{2}\right]+20nW\epsilon_1+\frac{1}{2}\lambda, \ \ \forall n \in [N]
$$

By Assumption 6,

$$
\begin{aligned}
\mathbb{E}_{\rho_{\mathrm{cov}}^{n}}\left[\left(\Delta f_{k}\right)^{2}\right] &=\mathbb{E}_{(s, a) \sim \rho_{\mathrm{cov}}^{n}}\left[\left(f_{k}^{*}(s,a)-f_{k}(s,a)\right)^{2}\right] \\ &\leq C \cdot\left(L\left(f_{k} ; \rho_{\mathrm{cov}}^{n}, Q_{b^{n}}^{k}-b^{n}\right)-L\left(f_{k}^{*} ; \rho_{\mathrm{cov}}^{n}, Q_{b^{n}}^{k}-b^{n}\right)\right)\\
& \leq C \cdot \epsilon_{\mathrm{stat}}\ \  \text{(by Lemma \ref{E.5})}
\end{aligned}
$$

By the choice of $\epsilon,\left\|\Delta f_{t}\right\|_{\mathcal{Z}^{n}}^{2} \leq 100 \epsilon,\ \forall n \in [N]$ with probability at least $1-\frac{1}{4}\delta$. Combining the above result with Lemma \ref{C.7}, we finish our proof of Lemma \ref{E.4}. 

Next, we give an explicit form of $\epsilon_{\text {stat }}$ as defined in the next lemma.
\end{proof}

\begin{lemma} \label{E.5}
(Concentration of statistical error).
Following the same notation as in Lemma \ref{E.4}, it holds with probability at least $1-\frac{1}{8}\delta$ that

$$
L\left(f_{k} ; \rho_{c o v}^{n}, Q_{b^{n}}^{k}-b^{n}\right)-L\left(f_{k}^{*} ; \rho_{c o v}^{n}, Q_{b^{n}}^{k}-b^{n}\right) \leq \frac{500 C \cdot W^{4} \cdot \log \left(\frac{\mathcal{N}\left(\mathcal{F}, \epsilon_{2}\right)}{\delta_3}\right)}{M}+13 W^{2} \cdot \epsilon_{2},
$$

where $C, \epsilon_{0}$ are defined in Lemma \ref{B.3}, and $\epsilon_{2}>0$ denotes the function cover radius which will be determined later. 
\end{lemma}

\begin{proof} 
This proof builds on \citet{feng2021provably}'s Lemma C.4, but deals with the concentration of importance sampling estimator.  First note that in the loss function, the expectation has a nested structure: the outer expectation is taken over $(s, a) \sim \rho_{\text {cov }}^{n}$ and the inner conditional expectation is $Q_{b^{n}}^{k}(s, a)=$ $\mathbb{E}^{\pi_{k}}\left[\sum_{h=0}^{\infty} \gamma^{h}\left(r\left(s_{h}, a_{h}\right)+b^{n}\left(s_{h}, a_{h}\right)\right) \mid\left(s_{0}, a_{0}\right)=(s, a)\right]$ given a sample of $(s, a) \sim \rho_{\text {cov }}^{n}$. To simplify the notation, we use $x$ to denote $(s, a), y \mid x$ for an unbiased sample of $Q_{b^{n}}^{k}(s, a)-b^{n}(s, a)$ through importance sampling, and $\nu$ for $\rho_{\text {cov }}^{n}$, the marginal distribution over $x$, then the loss function can be recast as

$$
\begin{aligned}
&\mathbb{E}_{x \sim \nu}\left[\left(f_{k}(x)-\mathbb{E}[y \mid x]\right)^{2}\right]:=L\left(f_{k} ; \rho_{\text {cov }}^{n}, Q_{b^{n}}^{k}-b^{n}\right) \\
&\mathbb{E}_{x \sim \nu}\left[\left(f_{k}^{*}(x)-\mathbb{E}[y \mid x]\right)^{2}\right]:=L\left(f_{k}^{*} ; \rho_{\text {cov }}^{n}, Q_{b^{n}}^{k}-b^{n}\right)
\end{aligned}
$$

In particular, $f_{k}$ can be rewritten as

$$
f_{k} \in \underset{f \in \mathcal{F}}{\operatorname{argmin}} \sum_{i=1}^{M}\left(f\left(x_{i}\right)-y_{i}\right)^{2},
$$

where $\left(x_{i}, y_{i}\right)$ are drawn i.i.d.: $x_{i}$ is generated following the marginal distribution $\nu$ and $y_{i}$ is generated conditioned on $x_{i}$. \\

Note that $y_i$ is collected by importance sampling estimator, which does not necessarily come from Monte Carlo sampling. However, in the latest time when the agent interacts with the environment, the samples are drawn i.i.d., which guaranteed the same property for the importance sampling process.

For any function $f$, we have:

$$
\begin{aligned}
& \mathbb{E}_{x, y}\left[\left(f_{k}(x)-y\right)^{2}\right] \\
=& \mathbb{E}_{x, y}\left[\left(f_{k}(x)-\mathbb{E}[y \mid x]\right)^{2}\right]+\mathbb{E}_{x, y}\left[(\mathbb{E}[y \mid x]-y)^{2}\right]+2 \mathbb{E}_{x, y}\left[\left(f_{k}(x)-\mathbb{E}[y \mid x]\right)(\mathbb{E}[y \mid x]-y)\right] \\
=& \mathbb{E}_{x, y}\left[\left(f_{k}(x)-\mathbb{E}[y \mid x]\right)^{2}\right]+\mathbb{E}_{x, y}\left[(\mathbb{E}[y \mid x]-y)^{2}\right],
\end{aligned}
$$

where the last step follows from the cross term being zero. Thus we can rewrite the generalization error as

$$
\begin{aligned}
& \mathbb{E}_{x}\left[\left(f_{k}(x)-\mathbb{E}[y \mid x]\right)^{2}\right]-\mathbb{E}_{x}\left[\left(f_{k}^{*}(x)-\mathbb{E}[y \mid x]\right)^{2}\right] \\
=& \mathbb{E}_{x, y}\left(f_{k}(x)-y\right)^{2}-\mathbb{E}_{x, y}\left(f_{k}^{*}(x)-y\right)^{2} .
\end{aligned}
$$

Next, we establish a concentration bound on $f_{k}$. Since $f_{k}$ depends on the training set $\left\{\left(x_{i}, y_{i}\right)\right\}_{i=1}^{M}$, as discussed in Lemma \ref{B.6}, we use a function cover on $\mathcal{F}$ for a uniform convergence argument. We denote by $\mathcal{F}_{k}^{n}$ the $\sigma$-algebra generated by randomness before epoch $n$ iteration $k$. Recall that $f_{k}^{*} \in \operatorname{argmin}_{f \in \mathcal{F}} L\left(f ; \rho_{\mathrm{cov}}^{n}, Q_{b^{n}}^{k}-b^{n}\right)$. Conditioning on $\mathcal{F}_{k}^{n}, \rho_{\mathrm{cov}}^{n}, Q_{b^{n}}^{k}-b^{n}$, and $f_{k}^{*}$ are all deterministic. For any $f \in \mathcal{F}$, we define

$$
Z_{i}(f):=\left(f\left(x_{i}\right)-y_{i}\right)^{2}-\left(f_{k}^{*}\left(x_{i}\right)-y_{i}\right)^{2}, \quad i \in[M]
$$

Then $Z_{1}(f), \ldots, Z_{M}(f)$ are i.i.d. random variables and notice that $y_i$ is drawn from importance sampling estimator. From Lemma \ref{E.3}, we know that with probability at least $1-M\delta_1$, $y_i\leq 2 G_{\text{max}}\leq W,\ i\in [M]$. 

Conditioned on this event, we have

$$
\begin{aligned}
\mathbb{V}\left[Z_{k}(f) \mid \mathcal{F}_{k}^{n}\right] & \leq \mathbb{E}\left[Z_{i}(f)^{2} \mid \mathcal{F}_{k}^{n}\right] \\
&=\mathbb{E}\left[\left(\left(f\left(x_{i}\right)-y_{i}\right)^{2}-\left(f_{k}^{*}\left(x_{i}\right)-y_{i}\right)^{2}\right)^{2} \mid \mathcal{F}_{k}^{n}\right] \\
&=\mathbb{E}\left[\left(f\left(x_{i}\right)-f_{k}^{*}\left(x_{i}\right)\right)^{2} \cdot\left(f\left(x_{i}\right)+f_{k}^{*}\left(x_{i}\right)-2 y_{i}\right)^{2} \mid \mathcal{F}_{k}^{n}\right] \\
& \leq 36 W^{4} \cdot \mathbb{E}\left[\left(f\left(x_{i}\right)-f_{k}^{*}\left(x_{i}\right)\right)^{2} \mid \mathcal{F}_{k}^{n}\right] \\
& \leq 36 W^{4} \cdot\left(C \cdot \mathbb{E}\left[Z_{i}(f) \mid \mathcal{F}_{k}^{n}\right]\right)
\end{aligned}
$$

where the last inequality is by Lemma \ref{B.3}. Next, we apply Bernstein's inequality on the function cover $\mathcal{C}\left(\mathcal{F}, \epsilon_{2}\right)$ and take the union bound. Specifically, with probability at least $1-\delta_3$, for all $g \in \mathcal{C}\left(\mathcal{F}, \epsilon_{2}\right)$,

$$
\begin{aligned}
& \mathbb{E}\left[Z_{i}(g) \mid \mathcal{F}_{k}^{n}\right]-\frac{1}{M} \sum_{i=1}^{M} Z_{i}(g) \\
& \leq \sqrt{\frac{2 \mathbb{V}\left[Z_{i}(g) \mid \mathcal{F}_{k}^{n}\right] \cdot \log \frac{\mathcal{N}\left(\mathcal{F}_{,}, \epsilon_{2}\right)}{\delta_3}}{M}}+\frac{12 W^{4} \cdot \log \frac{\mathcal{N}\left(\mathcal{F}_{,}, \epsilon_{2}\right)}{\delta_3}}{M} \\
& \leq \sqrt{\frac{72 W^{4}\left(C \cdot \mathbb{E}\left[Z_{i}(g) \mid \mathcal{F}_{k}^{n}\right]\right) \cdot \log \frac{\mathcal{N}\left(\mathcal{F}, \epsilon_{2}\right)}{\delta_3}}{M}}+\frac{12 W^{4} \cdot \log \frac{\mathcal{N}\left(\mathcal{F}, \epsilon_{2}\right)}{\delta_3}}{M} . 
\end{aligned}
$$

For $f_{t}$, there exists $g \in \mathcal{C}\left(\mathcal{F}, \epsilon_{2}\right)$ such that $\left\|f_{k}-g\right\|_{\infty} \leq \epsilon_{2}$ and

$$
\begin{aligned}
\left|Z_{i}\left(f_{k}\right)-Z_{i}(g)\right| &=\left|\left(f_{k}\left(x_{i}\right)-y_{i}\right)^{2}-\left(g\left(x_{i}\right)-y_{i}\right)^{2}\right| \\
&=\left|f_{k}\left(x_{i}\right)-g\left(x_{i}\right)\right| \cdot\left|f_{k}\left(x_{i}\right)+g\left(x_{i}\right)-2 y_{i}\right| \leq 6 W^{2} \epsilon_{2} .
\end{aligned}
$$

Therefore, with probability at least $1-\delta_3$,

$$
\begin{aligned}
& \mathbb{E}\left[Z_{i}\left(f_{k}\right) \mid \mathcal{F}_{k}^{n}\right]-\frac{1}{M} \sum_{i=1}^{M} Z_{i}\left(f_{k}\right) \\
&\leq \mathbb{E}\left[Z_{i}(g) \mid \mathcal{F}_{k}^{n}\right]-\frac{1}{M} \sum_{i=1}^{M} Z_{i}(g)+12 W^{2} \epsilon_{2} \\
&\leq  \sqrt{\frac{72 W^{4}\left(C \cdot \mathbb{E}\left[Z_{i}(g) \mid \mathcal{F}_{k}^{n}\right]\right) \log \frac{\mathcal{N}\left(\mathcal{F}, \epsilon_{2}\right)}{\delta_3}}{M}}+\frac{12 W^{4} \log \frac{\mathcal{N}\left(\mathcal{F},\epsilon_{2}\right)}{\delta_3}}{M}+12 W^{2} \epsilon_{2} \\
&\leq \sqrt{\frac{72 W^{4}\left(C \cdot \mathbb{E}\left[Z_{i}(f_k) \mid \mathcal{F}_{k}^{n}\right]+6CW^2\epsilon_2\right) \log \frac{\mathcal{N}\left(\mathcal{F}, \epsilon_{2}\right)}{\delta_3}}{M}}+\frac{12 W^{4} \log \frac{\mathcal{N}\left(\mathcal{F},\epsilon_{2}\right)}{\delta_3}}{M}+12 W^{2} \epsilon_{2}.
\end{aligned}
$$

Since $f_{k}$ is an empirical minimizer, we have $\frac{1}{M} \sum_{i=1}^{M} Z_{i}\left(f_{k}\right) \leq 0$. Thus,

$\mathbb{E}\left[Z_{i}\left(f_{k}\right) \mid \mathcal{F}_{k}^{n}\right] \leq \sqrt{\frac{72 W^{4}\left(C \cdot \mathbb{E}\left[Z_{i}\left(f_{k}\right) \mid \mathcal{F}_{k}^{n}\right]+6 C W^{2} \epsilon_{2}\right) \log \frac{\mathcal{N}\left(\mathcal{F}, \epsilon_{2}\right)}{\delta_3}}{M}}+\frac{12 W^{4} \log \frac{\mathcal{N}\left(\mathcal{F}, \epsilon_{2}\right)}{\delta_3}}{M}+12 W^{2} \epsilon_{2} .$

Solving the above inequality with quadratic formula and using $\sqrt{a+b} \leq \sqrt{a}+\sqrt{b}, \sqrt{a b} \leq a / 2+b / 2$ for $a>0, b>0$, we obtain

$$
\mathbb{E}\left[Z_{i}\left(f_{k}\right) \mid \mathcal{F}_{k}^{n}\right] \leq \frac{500 C \cdot W^{4} \cdot \log \frac{\mathcal{N}\left(\mathcal{F}, \epsilon_{2}\right)}{\delta_3}}{M}+13 W^{2} \cdot \epsilon_{2}
$$

Since the right-hand side is a constant, through taking another expectation, we have

$$
\mathbb{E}\left[Z_{i}\left(f_{k}\right)\right] \leq \frac{500 C \cdot W^{4} \cdot \log \frac{\mathcal{N}\left(\mathcal{F}, \epsilon_{2}\right)}{\delta_3}}{M}+13 W^{2} \cdot \epsilon_{2}.
$$

Notice that $\mathbb{E}\left[Z_{i}\left(f_{k}\right)\right]=L\left(f_{k} ; \rho_{\mathrm{cov}}^{n}, Q_{b^{n}}^{k}-b^{n}\right)-L\left(f_{k}^{*} ; \rho_{\mathrm{cov}}^{n}, Q_{b^{n}}^{k}-b^{n}\right)$. 

Finally, we let $(1-M\delta_1)(1-\delta_3)\geq 1-\frac{1}{8}\delta$ , so the desired result is obtained.
\end{proof}

\begin{lemma} \label{E.6}
(One-sided error).
With probability at least $1-\frac{\delta}{2}$ it holds that
\begin{equation}
	\begin{aligned}
	\forall n\in [N],\ \forall k\in\{0,\cdots,K-1\},\ \forall (s,a)\in\mathcal{K}^n: \ 0\leq Q_{b^n}^{k,*}(s,a)-\widehat{Q}_{b^n}^{k}(s,a)\leq2b_{\omega}^n(s,a)
	\end{aligned}
\end{equation}
\end{lemma}

\begin{proof}

When $(s,a)\in \mathcal{K}^n$,
$$
\widehat{Q}_{b^n}^{k}(s,a)=f_k(s,a)+b_{\omega}^n(s,a)
$$
$$
Q_{b^n}^{k,*}(s,a)=f_k^*(s,a)+b^n(s,a)=f_k^*(s,a)+2b_{\omega}^n(s,a)
$$
Then,
$$
|Q_{b^n}^{k,*}(s,a)-\widehat{Q}_{b^n}^{k}(s,a)-b_{\omega}^n(s,a)|=|f_k^*(s,a)-f_k(s,a)|=|\Delta f_k(s,a)|
$$
According to Lemma \ref{E.4}, with probability at least $1-\frac{1}{2} \delta\ $, $||\Delta f_k||_{\widehat{\mathcal{Z}}^n}^2\leq \epsilon,\  \forall n \in [N]$\\
Using the definition of $b_{\omega}^n(s,a)$, we have
$$
|\Delta f_k(s,a)|\leq \omega(\widehat{\mathcal{F}}^n,s,a)\leq b_{\omega}^n(s,a)\ ( \beta<1)
$$
Finally, Lemma \ref{E.6} concludes.
\end{proof}

\section{Limitation of Previous Implementations} \label{App F}
Note that we do not compare our method directly with implementations in \citep{agarwal2020pc, feng2021provably}, as we discovered some limitations presented in their implementations. We show our insights in this section and provide an empirical evaluation of the quality of implementations of our algorithm and previous ones.  

Observation normalization is also very crucial for on-policy algorithms, but it is missing in those implementations. For the MountainCar environment, we find that the difficulty is not from the exploration problem, but from the ill-shaped observation. In their experiments, \textbf{PPO}-based exploration algorithms take up to 10k episodes to learn a near-optimal policy in MountainCar environment, however, with a running mean-std observation normalization, it only takes PPO-based algorithms a few episodes to learn the task. 

Furthermore, both of their implementations strictly follow the theoretical algorithms and use a ``Roll-In'' mechanism in order to get the previous distribution $\rho$. Although a recent study \citep{li2022understanding} shows evidence of leveraging the ``Roll-In'' mechanism in single-task RL problems for the off-policy algorithms, it still remains unknown whether such mechanism benefits on-policy algorithms in single-task RL problems. In our experiment, we find that \textbf{PC-PG} or \textbf{ENIAC} with ``Roll-In'' does not provide efficiency compared to its counterpart variant. We hypothesize that it is because the stochasticity of \textbf{PPO} and the environment is enough for the policy itself to recover the state distribution, thus the additionally introduced ``Roll-In'' is not needed.

Additionally, experiments from previous works \citep{agarwal2020pc, feng2021provably} compared exploration capability with \textbf{RND}, the current state-of-the-art algorithm on Montezuma's Revenge \citep{bellemare13arcade, burda2018exploration}. However, we find there is some discrepancy between their implementation and the original implementation of \textbf{RND}. Most importantly, their implementation does not use next state $s'$ to determine the intrinsic reward of state action pair $(s, a)$. The reason why this is crucial is that using $s'$ to determine the intrinsic reward integrates the novelty of the $(s, a)$ while using $s$ will lose the information of the action.

To demonstrate our point, we tested the original implementation of \citep{agarwal2020pc, feng2021provably} on MountainCarContinuous with running observation normalization (for all running algorithms). With observation normalization, our implemented algorithms easily learn the task within 10000 steps, significantly better than  results reported in \citep{agarwal2020pc, feng2021provably}. Moreover, we also test their implementations along with observation normalization. The performance of their implementations does not improve much over the course of 10000 steps, which demonstrates our point that their ``Roll-In" mechanism may not provide efficiency.

Our implementations \citep{stable-baselines3}, including \textbf{RND} and \textbf{PPO}, succeed to find rewards in the environments, while implementations from previous works do not. The result is shown in Figure \ref{fig:2}.

\clearpage

\section{Hyperparameters} 
We implemented our method based on the open source package \citep{stable-baselines3}, and the performance of \textbf{PPO} is obtained by running  the built-in implemented \textbf{PPO}. Following \citep{burda2018exploration}, we use smaller batch size (compared to 64 in standard MuJoCo environment \citep{schulman2017proximal}), specifically 32 in SparseHopper and 16 in SparseWalker2d and SparseHalfCheetah. The detailed hyperparameters are showed in the table \ref{hypertable}.

\begin{table}[h!]\label{hypertable}
\centering
 \begin{tabular}{||c c c||} 
 \hline
 Hyperparameter & \textbf{Value} (\textbf{LPO}, \textbf{ENIAC}) & \textbf{Value} (\textbf{PPO})\\ [0.5ex] 
 \hline\hline
 $N$ & 2048 & 2048 \\ 
 $T$ & 2e6 & 2e6 \\
 $\lambda$ & 0.95 & 0.95\\
 $\gamma^{(int)}$ & 0.999 & -\\
 $\gamma^{(ext)}$ & 0.99 & 0.99\\
 $\alpha$ & 2 & -\\
 $\beta$ & 1 & - \\ 
 Learning rate & 1e-4 & 1e-4 \\
 Batch size & 32, 16 & 32, 16 \\
 Number of epoch per iteration & 10 & 10 \\
 [1ex] 
 \hline
 \end{tabular}
\end{table}

\end{document}